\documentclass{article}
\usepackage{fullpage}

\usepackage{hyperref}
\usepackage[authoryear]{natbib}
\usepackage{float}
\usepackage{subfig}
\usepackage{graphicx}
\usepackage{theorem}
\usepackage{afterpage}
\usepackage{multirow}
\usepackage{needspace}
\usepackage{amsmath,amssymb}
\usepackage{url}
\usepackage{citesort}
\usepackage{epsfig}
\usepackage{color}
\usepackage{url}
\usepackage{bbm}
\usepackage{tikz}
\usepackage{wrapfig}
\usepackage[margin=1in]{geometry}
\usepackage{epsfig}
\usepackage{epstopdf}
 
\usepackage{setspace}
\onehalfspace
\usepackage[linesnumbered,ruled,resetcount,vlined]{algorithm2e}
\allowdisplaybreaks

\newtheorem{theorem}{Theorem}
\newtheorem{lemma}{Lemma}
\newtheorem{claim}{Claim}

\newtheorem{corollary}{Corollary}

\newtheorem{definition}{Definition}

\newtheorem{fact}{Fact}

\newcommand{\eps}{\epsilon}
\newcommand{\abs}[1]{\left| #1 \right|}

\newcommand{\bE}{\mathbf{E}}
\newcommand{\var}{\mathbf{Var}}
\renewcommand{\Pr}{\mathbf{Pr}}

\newcommand{\A}{\mathcal{A}}

\newcommand{\K}{\tilde{K}}

\renewcommand{\paragraph}[1]{\medskip \noindent {\bf #1.}}

\newenvironment{proof}{\trivlist\item[]\emph{Proof}:}%
{\unskip\nobreak\hskip 1em plus 1fil\nobreak$\Box$
\parfillskip=0pt%
\endtrivlist}

\definecolor{DSgray}{cmyk}{0,1,0,0}
\newcommand{\Authornote}[2]{{\small\textcolor{DSgray}{$<${  #1: #2
}$>$}}}

\newcommand{\xnote}[1]{{\Authornote{Xi}{#1}}}

\renewcommand{\paragraph}[1]{\medskip \noindent {\bf #1.}}

\newcommand{\two}{coin-tossing}
\newcommand{\arm}{$\eps$-top-$K$ arms}
\newcommand{\armT}{$\eps'$-top-$K$ arms}
\newcommand{\armG}{$\gamma$-top-$K$ arms}
\newcommand{\vb}{\mathbf{b}}
\newcommand{\vB}{\mathbf{B}}
\newcommand{\D}{\mathcal{D}}

\newcommand{\tDelta}{\tilde{\Delta}}
\newcommand{\ttheta}{\tilde{\theta}}

\newcommand{\tmu}{\tilde{\mu}}

\newcommand{\tp}{\tilde{p}}
\newcommand{\gap}{\Delta}
\let\tilde\widetilde

\begin{document}

\title{Adaptive Multiple-Arm Identification\footnote{Author names are listed in alphabetical order. Preliminary version to appear in ICML 2017.}
}

 \author{
 Jiecao Chen\\ Computer Science Department\\ Indiana University at Bloomington\\ {\tt jiecchen@umail.iu.edu} \and
 Xi Chen\\ Leonard N. Stern School of Business\\ New York University\\ {\tt xchen3@stern.nyu.edu} \and
 Qin Zhang\\ Computer Science Department\\ Indiana University at Bloomington\\ {\tt qzhangcs@indiana.edu} \and
 Yuan Zhou\\ Computer Science Department\\ Indiana University at Bloomington\\ {\tt yzhoucs@indiana.edu}
 }
\maketitle

\abstract{
We study the problem of selecting $K$ arms with the highest expected rewards in a stochastic $n$-armed bandit game. This problem has a wide range of applications, e.g., A/B testing, crowdsourcing, simulation optimization. Our goal is to develop a PAC algorithm, which, with probability at least $1-\delta$, identifies a set of $K$ arms with the aggregate regret at most $\epsilon$. The notion of aggregate regret for multiple-arm identification was first introduced in \cite{Zhou:14} , which is defined as the difference of the averaged expected rewards between the selected set of arms and the best $K$ arms. In contrast to \cite{Zhou:14} that only provides instance-independent sample complexity, we introduce a new hardness parameter for characterizing the difficulty of any given instance. We further develop two algorithms and establish the corresponding sample complexity in terms of this hardness parameter. The derived sample complexity can be significantly smaller than  state-of-the-art results for a large class of instances and matches the instance-independent lower bound upto a $\log(\epsilon^{-1})$ factor in the worst case. We also prove a lower bound result showing that the extra $\log(\epsilon^{-1})$ is necessary for instance-dependent algorithms using the introduced hardness parameter.
}

\section{Introduction}
\label{sec:intro}

Given a set of alternatives with different quality,  identifying high quality alternatives via a sequential experiment is an important problem in multi-armed bandit (MAB) literature, which is also known as the ``pure-exploration" problem. This problem has a wide range of applications. For example, consider the A/B/C testing problem with multiple website designs, where each candidate design corresponds to an alternative. In order to select high-quality designs, an agent could display different designs to website visitors and measure the attractiveness of an design. The question is: how should the agent adaptively select which design to be displayed next so that the high-quality designs can be quickly and accurately identified? For another example, in crowdsourcing, it is critical to identify high-quality workers from a pool of a large number of noisy workers. An effective strategy is testing workers by gold questions, i.e., questions with the known answers provided by domain experts. Since the agent has to pay  a fixed monetary reward for each answer from a worker, it is important to implement a cost-effective strategy for to select the top workers with the minimum number of tests. Other applications include simulation optimization, clinical trials, etc.

More formally, we assume that there are $n$ alternative arms, where the $i$-th arm is associated with an unknown reward distribution $\mathcal{D}_i$ with mean $\theta_i$.  For the ease of illustration, we assume each $\mathcal{D}_i$ is supported on $[0,1]$. In practice, it is easy to satisfy this assumption by a proper scaling. For example, the traffic of a website or the correctness of an answer for a crowd worker (which simply takes the value either 0 or 1), can be scaled to $[0,1]$. The mean reward $\theta_i$ characterizes the quality of the $i$-th alternative. The agent sequentially pulls an arm, and upon each pulling of  the $i$-th arm, the \emph{i.i.d.} reward from $\mathcal{D}_i$ is observed. The goal of ``top-$K$ arm identification'' is to design an adaptive arm pulling strategy so that the top $K$ arms with the largest mean rewards can be identified with the minimum number of trials.  In practice, identifying the exact top-$K$ arms usually requires a large number of arm pulls, which could be wasteful. In many applications (e.g., crowdsourcing), it is  sufficient to find an ``\emph{approximate set}" of top-$K$ arms. To measure the quality of the selected arms, we adopt the notion of \emph{aggregate regret} (or regret for short) from \cite{Zhou:14}. In particular, we assume that arms are ordered by their mean $\theta_1 \geq \theta_2 \geq \dots \geq \theta_n$ so that the set of the best $K$ arms is $\{1,\ldots, K\}$. For the selected arm set $T$ with the size $|T|=K$, the aggregate regret  $\mathcal{R}_T$ is defined as,
\begin{equation}\label{eq:regret}
 \mathcal{R}_T=\frac{1}{K} \left( \sum_{i=1}^K \theta_i - \sum_{i \in T} \theta_i \right).
\end{equation}
The set of arms $T$ with the aggregate regret less than a pre-determined tolerance level $\epsilon$ (i.e. $\mathcal{R}_T \leq \epsilon$) is called  \arm. In this paper, we consider the $\epsilon$-top-K-arm problem in the ``fixed-confidence"  setting: given a target confidence level $\delta>0$, the goal is to find a set of \arm\ with the probability at least  $1-\delta$. This is also known as the PAC (probably approximately correct) learning setting. We are interested in achieving this goal with as few arm pulls (sample complexity) as possible.

To solve this problem, \cite{Zhou:14} proposed the OptMAI algorithm and established its sample complexity $\Theta \left(\frac{n}{\epsilon^2} \left(1+ \frac{\ln \delta^{-1}}{K}\right) \right)$, which is shown to be asymptotically optimal. However, the algorithm and the corresponding sample complexity in \cite{Zhou:14} are \emph{non-adaptive} to the underlying instance. In other words, the algorithm does not utilize the information obtained in known samples to adjust its future sampling strategy; and as a result, the sample complexity only involves the parameters $K$, $n$, $\delta$ and $\epsilon$ but is independent of $\{\theta_i\}_{i=1}^n$. \cite{Chen-Lin-King-Lyu-Chen-13} developed the CLUCB-PAC algorithm and established an instance-dependent sample complexity for a more general class of problems, including the $\epsilon$-top-K arm identification problem as one of the key examples. When applying the CLUCB-PAC algorithm to identify \arm, the sample complexity becomes $O((\log H^{(0,\epsilon)} + \log \delta^{-1})H^{(0,\epsilon)})$ where $H^{(0,\epsilon)} =  \sum_{i=1}^{n} \min\{(\Delta_i)^{-2}, \epsilon^{-2} \}$, $\Delta_i = \theta_i - \theta_{K+1}$ for $i \leq K$, $\Delta_i = \theta_K - \theta_i$ for $i > K$. The reason why we adopt the notation $H^{(0,\epsilon)}$ will be clear from Section \ref{sec:intro-main-results}.  However, this bound may be improved for the following two reasons. { First, intuitively, the hardness parameter $H^{(0,\epsilon)}$ is the total number of necessary pulls needed for each arm to identify whether it is among the top-$K$ arms or the rest so that the algorithm can decide whether to accept or reject the arm (when the arm's mean is $\epsilon$-close to the boundary between the top-$K$ arms and the rest arms, it can be either selected or rejected). However, in many cases, even if an arm's mean is $\eps$-far from the boundary, we may still be vague about the comparison between its mean and the boundary, i.e. either selecting or rejecting the arm satisfies the aggregate regret bound.}
This may lead to fewer number of pulls and a smaller hardness parameter for the same instance. Second, the worst-case sample complexity for CLUCB-PAC becomes $O((\log n + \log \eps^{-1} + \log \delta^{-1})n \epsilon^{-2})$. When $\delta$ is a constant, this bound is $\log n$ times more than the best non-adaptive algorithm in \cite{Zhou:14}.

In this paper, we explore towards the above two directions  and introduce new instance-sensitive algorithms for the problem of identifying \arm. These algorithms significantly improve the sample complexity by CLUCB-PAC for many common instances and almost match the best non-adaptive algorithm in the worst case.

Specifically, we first introduce a new parameter $H$ to characterize the hardness of a given instance. This new hardness parameter $H$ could be smaller than the hardness parameter $\tilde{H}$ used in the literature, in many natural instances. For example, we show in Lemma~\ref{lem:c-spread-bound}  that when $\{\theta_i\}_{i=1}^n$ are sampled from a continuous distribution with bounded probability density function (which is a common assumption in Bayesian MAB and natural for many applications), for $K = \gamma n$ with $\gamma \leq 0.5$, our hardness parameter $H=O(n/\sqrt{\epsilon})$ while $\tilde{H}=\Omega(n/\epsilon)$.

 Using this new hardness parameter $H$, we first propose an easy-to-implement algorithm-- {\sc AdaptiveTopK} and relate its sample complexity to $H$. In Theorem~\ref{thm:intro-basic}, we show that  {\sc AdaptiveTopK}  uses $O\left(\left(\log\log(\epsilon^{-1})+\log n + \log \delta^{-1}\right) H\right)$ to identify \arm~ with probability at least $1 - \delta$. Note that this bound has a similar form  as the one in \cite{Chen-Lin-King-Lyu-Chen-13}, but as mentioned above, we have an $\sqrt{\eps}$ -factor improvement in the hardness parameter for those instances where Lemma~\ref{lem:c-spread-bound} applies.

We then propose the second algorithm ({\sc ImprovedTopK}) with even less sample complexity, which removes the $\log n$ factor in the sample complexity. In Theroem~\ref{thm:intro-improved}, we show that the algorithm uses $O\left(\left(\log\epsilon^{-1}+ \log \delta^{-1}\right) H\right)$ pulls to identify \arm~ with probability $1-\delta$.  Since $H$ is always $\Omega(n/\epsilon^{2})$ (which will be clear when the $H$ is defined in Section~\ref{sec:intro-main-results}), the worst-case sample complexity of {\sc ImprovedTopK} matches the best instance-independent shown in \cite{Zhou:14} up to an extra $\log(\epsilon^{-1})$ factor (for constant $\delta$).  We are also able to show that this extra $\log(\epsilon^{-1})$ factor is a necessary expense by being instance-adaptive (Theorem~\ref{thm:intro-lb-eps}). It is also noteworthy that as a by-product of establishing {\sc ImprovedTopK}, we developed an algorithm that approximately identifies the $k$-th best arm, which may be of independent interest. Please see Algorithm \ref{alg:k-th} for details.

We are now ready to introduce our new hardness parameters and summarize the main results in technical details.

\subsection{Summary of Main Results}
\label{sec:intro-main-results}
Following the existing literature (see, e.g., \cite{Bubeck:13}), we first define the gap of the $i$-th arm
\begin{eqnarray}\label{eq:gap}
\Delta_i(K) = \begin{cases}
  \theta_i - \theta_{K+1} \quad & \text{if} \; i \leq K \\
  \theta_K- \theta_i      \quad & \text{if} \; i \geq K+1.
\end{cases}
\end{eqnarray}
Note that when $K = 1$, $\Delta_i(K)$ becomes $\theta_1 - \theta_i$ for all $i \geq 2$ and $\Delta_1(K)=\theta_1-\theta_2$. 
When $K$ is clear from the context, we simply use $\Delta_i$ for $\Delta_i(K)$. One commonly used hardness parameter for quantifying the sample complexity in the existing literature (see, e.g., \cite{Bubeck:13, Karnin:13}) is $\tilde{H} \triangleq \sum_{i=1}^n \Delta_i^{-2}$.  If there is an extremely small gap $\Delta_i$, the value of $\tilde{H}$ and thus the corresponding sample complexity can be super large. This hardness parameter is natural when the goal is to identify the exact top-$K$ arms, where a sufficient gap between an arm and the boundary (i.e. $\theta_K$ and $\theta_{K+1}$) is necessary. However, in many applications (e.g., finding high-quality workers in crowdsourcing), it is an overkill to select the exact top-$K$ arms. For example, if all the top-$M$ arms with $M > K$ have very close means, then any subset of them of size $K$ forms an $\epsilon$-top-$K$ set in terms of the aggregate regret in \eqref{eq:regret}. Therefore, to quantify the sample complexity when the metric is the aggregate regret, we need to construct a new hardness parameter.

Given $K$ and an error bound $\epsilon$, let us define $t = t(\epsilon, K)$ to be the largest $t \in \{0, 1, 2, \dots, K - 1\}$ such that
\begin{align}\label{eq:def-t-eps-K}
\Delta_{K-t} \cdot t \leq K\epsilon \qquad \mbox{and}\qquad \Delta_{K+t+1}\cdot t \leq K \epsilon.
\end{align}
 Note that $\Delta_{K-t} \cdot t = (\theta_{K-t}-\theta_{K+1}) \cdot t$ upper-bounds the total gap of the $t$ worst arms in the top $K$ arms and $\Delta_{K+t+1} \cdot t=(\theta_{K}-\theta_{K+t+1})  \cdot t$ upper-bounds the total gap of the $t$ best arms in the non-top-$K$ arms. Intuitively, the definition in \eqref{eq:def-t-eps-K} means that we can tolerate  exchanging at most $t$ best arms in the non-top-$K$ arms with the $t$ worst arms in the top-$K$ arms. 

Given $t = t(\epsilon, K)$, we define
\begin{align}
\label{eq:def-Psi-t}
\Psi_t = \min(\Delta_{K-t}, \Delta_{K+t+1}),
\end{align}
and
\begin{align}
\label{eq:def-Psi-t-eps}
\Psi_t^{\epsilon} = \max(\epsilon, \Psi_t) .
\end{align}
We now introduce the following parameter to characterize the hardness of a given instance,
\begin{align}
H = H^{(t, \epsilon)} = \sum_{i=1}^{n} \min\{(\Delta_i)^{-2}, (\Psi_t^{\epsilon})^{-2}\} .
\end{align}
It is worthwhile to note that in this new definition of hardness parameter, no matter how small the gap $\Delta_i$ is, since  $\Psi_t^{\epsilon} \geq \epsilon$,  we always have $H^{(t,\epsilon)} \leq n\epsilon^{-2}$. We also note that since $\Psi_{t}$ is non-decreasing in $t$, $H^{(t, \epsilon)}$ is non-increasing in $t$.

Our first result is an easy-to-implement algorithm (see Algorithm \ref{alg:basic}) that identifies $\epsilon$-top-$K$ arms with sample complexity related to $H^{(t, \epsilon)}$.

\begin{theorem}
  \label{thm:intro-basic}
There is an algorithm that computes \arm\ with probability at least $(1 - \delta)$, and pulls the arms at most
$O\left(\left(\log\log \epsilon^{-1} + \log n + \log\delta^{-1}\right)H^{(t, \epsilon)} \right)$
times.
\end{theorem}

We also develop a more sophisticated algorithm (see Algorithm \ref{alg:improved}) with an improved sample complexity.

\begin{theorem}
\label{thm:intro-improved}

There is an algorithm that computes \arm\ with probability at least $(1 - \delta)$, and pulls the arms at most $O \left( \left(\log \eps^{-1} + \log  \delta^{-1} \right) H^{(t, \epsilon)} \right)$ times.
\end{theorem}

Since $\Psi_t^\epsilon \geq \epsilon$ and $H^{(t, \epsilon)} \leq {n}\epsilon^{-2}$, the worst-case sample complexity by Theorem~\ref{thm:intro-improved} is $O \left(\frac{n}{\epsilon^2} \left(\log \eps^{-1} + \log  \delta^{-1} \right)  \right)$. While the asymptotically optimal instance-independent sample complexity is $\Theta\left(\frac{n}{\epsilon^2} \left(1+ \frac{\ln \delta^{-1}}{K}\right) \right)$, we show that the $\log \epsilon^{-1}$ factor in Theorem~\ref{thm:intro-improved} is necessary for instance-dependent algorithms using $H^{(t, \epsilon)}$ as a hardness parameter. In particular, we prove the following lower-bound result.

\begin{theorem}
\label{thm:intro-lb-eps}
For any $n, K$ such that $n = 2K$, and any $\epsilon = \Omega(n^{-1})$, there exists an instance on $n$ arms so that $H^{(t, \epsilon)} = \Theta(n)$ and it requires $\Omega(n \log \eps^{-1})$ pulls to identify a set of $\epsilon$-top-$K$ arms with probability at least $0.9$.
\end{theorem}

Note that since $H^{(t, \eps)} = \Theta(n)$ in our lower bound instances, our Theorem~\ref{thm:intro-lb-eps} shows that the sample complexity has to be at least $\Omega(H^{(t, \eps)} \log \eps^{-1})$ in these instances. In other words, our lower bound result shows that for any instance-dependent algorithm, and any $\epsilon = \Omega(n^{-1})$, there exists an instance where sample complexity has to be $\Omega(H^{(t, \eps)} \log \eps^{-1})$. While Theorem~\ref{thm:intro-lb-eps} shows the necessity of the $\log \epsilon^{-1}$ factor in Theorem~\ref{thm:intro-improved}, it is not a lower bound for \emph{every} instance of the problem.

%

\subsection{Review of and Comparison with Related Works}
\label{sec:related}

The problem of identifying the single best arm (i.e. the top-K arms with $K=1$), has been studied extensively \citep{Even-Dar:02,Mannor:04,Audibert:10,Gabillon:11,Gabillon:12, Karnin:13,Jamieson:14,Kaufmann:16,Russo:16,ChenLiQiao:16}.  More specifically, in the special case when $K = 1$, our problem reduces to identifying an $\epsilon$-best arm, i.e. an arm whose expected reward is different from the best arm by an additive error of at most $\epsilon$, with probability at least $(1 - \delta)$. For this problem, \cite{Even-Dar:06} showed an algorithm with an instance-independent sample complexity $O\left(\frac{n}{\epsilon^2} \log \delta^{-1}\right)$ (and this was proved to be asymptotically optimal by \cite{Mannor:04}). An instance-dependent  algorithm for this problem was  given by \cite{Bubeck:13} and an improved algorithm was given by \cite{Karnin:13} with an instance-dependent sample complexity of $O\left(\sum_{i=2}^{n} \max\{\Delta_i, \epsilon\}^{-2} (\log \delta^{-1} + \log\log \max\{\Delta_i, \epsilon\}^{-1})\right) $. In the worst case, this bound becomes $O\left(\frac{n}{\epsilon^2} (\log \delta^{-1} + \log\log \epsilon^{-1}) \right)$, almost matching the instance-independent bound in \cite{Even-Dar:06}. When $K=1$, we have $t(\epsilon,K)=0$ and thus $H^{(t, \epsilon)} = H^{(0, \epsilon)} = \Theta\left(\sum_{i=2}^{n} \max\{\Delta_i, \epsilon\}^{-2}\right)$. Therefore, the sample complexity in our Theorem~\ref{thm:intro-improved} becomes $O( (\log \epsilon^{-1} + \log \delta^{-1}) H)= O\left(\frac{n}{\epsilon^2} (\log \epsilon^{-1} + \log \delta^{-1})\right)$ in the worst-case, almost matching the bound by \cite{Karnin:13}.

For the problem of identifying top-$K$ arms with $K > 1$, different notions of $\epsilon$-optimal solution have been proposed. One popular metric is the misidentification probability ({\sc MisProb}), i.e. $\Pr(T \neq \{1,\ldots, K\})$. In the PAC setting (i.e. controlling {\sc MisProb} less  than $\epsilon$ with probability at least $1-\delta$), many algorithms have been developed recently, e.g., \cite{Bubeck:13} in the fixed budget setting and \cite{Chen-Lin-King-Lyu-Chen-13} for both fixed confidence and fixed budget settings. \cite{Gabillon:16} further improved the sample complexity in \cite{Chen-Lin-King-Lyu-Chen-13}; however the current implementations of their algorithm have an exponential running time. As argued in \cite{Zhou:14}, the {\sc MisProb} requires to identify the exact top-$K$ arms, which might be too stringent for some applications (e.g., crowdsourcing). The {\sc MisProb} requires a certain gap between $\theta_K$ and $\theta_{K+1}$ to identify the top-$K$ arms, and this requirement is not unnecessary when using the aggregate regret. As shown in \cite{Zhou:14}, when the gap of any consecutive pair between $\theta_i$ and $\theta_{i+1}$ among the first $2K$ arms is $o(1/n)$, the sample complexity has to be huge ($\omega(n^2)$) to make the {\sc MisProb} less than $\epsilon$, while any $K$ arms among the first $2K$ form a desirably set of  $\epsilon$-top-$K$ arms in terms of aggregate regret. Therefore, we follow \cite{Zhou:14} and adopt the aggregate regret to define the approximate solution in this paper.

\cite{Kalyanakrishnan:12} proposed the so-called {\sc Explore-$K$} metric, which requires for each arm $i$ in the selected set $T$ to satisfy $\theta_i \geq \theta_K-\eps$ , where $\theta_K$ is the mean of the $K$-th best arm.  \cite{Cao:15} proposed a more restrictive notion of optimality---{\sc Elementwise-$\epsilon$-Optimal}, which requires the mean reward of the $i$-th best arm in the selected set $T$ be at least $\theta_i -\epsilon$ for $1 \leq i \leq K$. It is clear that the {\sc Elementwise-$\epsilon$-Optimal} is a stronger guarantee than    our $\epsilon$-top-$K$ in regret, while the latter is stronger than {\sc Explore-$K$}. \cite{Chen:16:COLT} further extended \cite{Cao:15} to pure exploration problems under matroid constraints. \cite{Audibert:10} and \cite{Bubeck:13} considered expected aggregate regret (i.e.  $\frac{1}{K} \left( \sum_{i=1}^K \theta_i - \bE\left(\sum_{i \in T} \theta_i \right) \right)$, where the expectation is taken over the randomness of the algorithm. Note that this notion of expected aggregate regret is a weaker objective than the aggregate regret.

Moreover, there are some other recent works studying the problem of best-arm identification in different setups, e.g., linear contextual bandit \citep{Soare:14}, batch arm pulls \citep{Jun:16}.

For our $\eps$-top-$K$ arm problem, the state-of-the-art instance-dependent sample complexity was given by \cite{Chen-Lin-King-Lyu-Chen-13} (see Section B.2 in Appendix of their paper). More specifically,  \cite{Chen-Lin-King-Lyu-Chen-13} proposed CLUCB-PAC algorithms that finds \arm\ with probability at least $(1-\delta)$ using $O\left( \left(\log \delta^{-1} + \log H^{(0, \epsilon)}\right) H^{(0, \epsilon)}\right)$ pulls.
Since we always have $H^{(0, \epsilon)} \geq H^{(t, \epsilon)} \geq \Omega(n)$ and $H^{(0, \epsilon)} \geq (\Psi_t^{\epsilon})^{-2}$, our Theorem~\ref{thm:intro-basic} is not worse than the bound in \cite{Chen-Lin-King-Lyu-Chen-13}. Indeed, in many common settings, $H^{(t, \epsilon)}$ can be much smaller than $H^{(0, \epsilon)}$ so that Theorem~\ref{thm:intro-basic} (and therefore Theorem~\ref{thm:intro-improved}) requires much less sample complexity. We explain this argument in more details as follows.

In many real-world applications, it is common to assume the arms $\theta_i$ are sampled from a prior distribution $\mathcal{D}$ over $[0, 1]$ with cumulative distribution function $F_\mathcal{D}(\theta)$. In fact, this is the most fundamental assumption in Bayesian multi-armed bandit literature (e.g., best-arm identification in Bayesian setup \cite{Russo:16}). In crowdsourcing applications, \cite{Chen15JMLR} and \cite{Abbasi15} also made this assumption for modeling workers' accuracy, which correspond to the expected rewards. Under this assumption, it is natural to let $\theta_i$ be the $(1-\frac{i}{n})$ quantile of the distribution $\mathcal{D}$, i.e. $F_{\mathcal{D}}^{-1}(1-\frac{i}{n})$. If the prior distribution $\mathcal{D}$'s probability density function $f_\mathcal{D} =  \frac{\mathrm{d} F_{\mathcal{D}} }{\mathrm{d} \theta}$ has bounded value (a few common examples include uniform distribution over $[0, 1]$, Beta distribution, or the truncated Gaussian distribution), the arms' mean rewards $\{\theta_i\}_{i=1}^n$ can be characterized by the following property with $c = O(1)$.

\begin{definition}
We call a set of $n$ arms $\theta_1 \geq \theta_2 \geq \dots \geq \theta_n$ \emph{$c$-spread} (for some $c \geq 1$) if for all $i, j \in [n]$ we have $|\theta_i - \theta_j| \in \left[\frac{|i-j|}{cn}, \frac{c |i-j|}{n} \right]$.
\end{definition}

The following lemma upper-bounds $H^{(t, \epsilon)}$ for $O(1)$-spread arms, and shows the improvement of our algorithms compared to \cite{Chen-Lin-King-Lyu-Chen-13} on $O(1)$-spread arms. 

\begin{lemma}
\label{lem:c-spread-bound}
Given a set of $n$ $c$-spread arms, let $K = \gamma n \leq \frac{n}{2}$. When $c = O(1)$ and $\gamma = \Omega(1)$, we have $H^{(t, \epsilon)} = O(n/\sqrt{\epsilon})$. In contrast, $H^{(0, \epsilon)} = \Omega(n/\epsilon)$ for $O(1)$-spread arms and every $K \in [n]$.
\end{lemma}

\begin{proof}
Given a set of $n$ $c$-spread arms, we have $\frac{t+1}{c n} \leq \Delta_{K-t}  \leq \frac{c(t+1)}{n}$ and  $\frac{t+1}{c n} \leq \Delta_{K+t+1}  \leq \frac{c(t+1)}{n}$. Therefore $t = t(\epsilon, K) \in [\sqrt{Kn\epsilon/c}, \sqrt{cKn\epsilon} - 1] =  [\sqrt{\gamma\epsilon/c} n, \sqrt{c\gamma\epsilon} n - 1]$, and $\Psi_t \geq \frac{t+1}{cn} \geq \sqrt{\gamma \epsilon/c^3}$. Therefore
\begin{align*}
H^{(t, \epsilon)} &\leq O(1) \sum_{i= 1}^{n} \min\left\{ \frac{c i}{n}, \Psi_t \right\}^{-2}  \leq O\left(t \cdot \Psi_t^{-2} + \sum_{i = t+1}^{n} \left(\frac{i}{cn}\right)^{-2}\right) \\
&= O\left(t \cdot \Psi_t^{-2} +c^2 n^2 / t\right)  = O(\sqrt{c\gamma\epsilon} n) \cdot \frac{c^3 }{\gamma \epsilon} + O\left(\frac{c^2n}{ \sqrt{\gamma \epsilon/c}}\right) = O(c^{3.5}\gamma^{-0.5}) \cdot \frac{n}{\sqrt{\epsilon}} .
\end{align*}
One the other hand, we have
\[
H^{(0, \eps)} \geq \sum_{i = 1}^{n - K} \min\{\Delta_{i+K}^{-2}, \eps^{-2}\} \geq \sum_{i = 1}^{n/2} \min\left\{\frac{n^2}{c^2 i^2}, \eps^{-2}\right\} = \sum_{i=1}^{[\epsilon n/c]} \eps^{-2} + \sum_{[\epsilon n/c] + 1}^{n/2} \frac{n^2}{c^2i^2} = \Omega\left(\frac{n}{c\epsilon}\right).
\]
\end{proof}

%



\section{An Instance Dependent Algorithm for $\epsilon$-top-$K$ Arms}
\label{sec:basic}

\begin{algorithm}[t]
\KwIn{$n$: number of arms; $K$ and $\eps$: parameters in $\eps$-top-$K$ arms; $\delta$: error probability}
\KwOut{\arm}
\DontPrintSemicolon
Let $r$ denote the current round, initialized to be $0$. Let $S_r \subseteq [n]$ denote the set of candidate arms at round $r$. $S_1$ is initialized to be $[n]$.
 Set $A, B \gets \emptyset$\;
$\gap \gets 2^{-r}$\; 
\While{$2 \cdot \gap \cdot (K - |A|) > \eps K$}{
  $r \gets r + 1$\; 
  Pull each arm in $S_r$ by $\Delta^{-2}\ln \frac{2 nr^2}{\delta}$ times, and let $\ttheta_i^r$ be the empirical-mean\;\label{line:delta-used}
  Define $\ttheta_a(S_r)$ and $\ttheta_{b}(S_r)$ be the $(K - |A| + 1)$th and $(K-|A|)$th largest empirical-means in $S_r$, and define
  \begin{equation}
    \label{eq:tDelta}
    \tDelta_i(S_r) = \max\left(\ttheta_i^r - \ttheta_{a}(S_r), \ttheta_{b}(S_r) - \ttheta_i^r\right)
  \end{equation}

  \While{$\max_{i\in S_r}\tDelta_i(S_r) > 2 \cdot \gap$}{ \label{line:while-start}
    $x \gets \arg\max_{i\in S_r} \tDelta_i(S_r)$\;
    \If{$\ttheta_{x}^r > \ttheta_a(S_r)$}{
      $A \gets A \cup \{x\}$\; \label{line:accept}
    }\Else{
      $B \gets B \cup \{x\}$\;
    }
    $S_r \gets S_r\backslash \{x\}$\;
  }
  $S_{r+1} \gets S_r$\;
  $\gap \gets 2^{-r}$\;
} \label{line:while-end}
Set $A'$ as the $(K - |A|)$ arms with the largest empirical-means in $S_{r+1}$\; \label{line:final}
\Return{$A\cup A'$}
\caption{{\sc AdaptiveTopK}$(n, \eps, K, \delta)$}
\label{alg:basic}
\end{algorithm}

In this section, we show Theorem~\ref{thm:intro-basic} by proving the following theorem.
\begin{theorem}
  \label{thm:basic}

Algorithm~\ref{alg:basic} computes \arm\ with probability at least $1 - \delta$, and pulls the arms at most
\[
O\left(\left(\log\log (\Delta_t^{\eps})^{-1} + \log n + \log\delta^{-1}\right)\sum_{i=1}^{n} \min\{(\Delta_i)^{-2}, (\Delta_t^\eps)^{-2}\}  \right)
\]
times, where $t\in\{0, 1, 2, \dots, K-1\}$ is the largest integer satisfying $\Delta_{K - t} \cdot t \leq K\eps$, and $\Delta_t^{\eps} = \max(\eps, \Delta_{K-t})$.
\end{theorem}

Note that Theorem~\ref{thm:basic} implies Theorem~\ref{thm:intro-basic} because of the following reasons: 1) $t$ defined in Theorem~\ref{thm:basic} is always at least $t(\epsilon, K)$ defined in \eqref{eq:def-t-eps-K}; and 2) $\Delta_t^{\epsilon} \geq \Psi_t^{\epsilon} \geq \epsilon$.

Algorithm~\ref{alg:basic} is similar to the accept-reject types of algorithms in e.g.\ \cite{Bubeck:13}. The algorithm goes by rounds for $r = 1, 2, 3, \dots$, and keeps at set of undecided arms $S_r \subseteq [n]$ at Round $r$. All other arms (in $[n] \setminus S_r$) are either accepted (in $A$) or rejected (in $B$). At each round, all undecided arms are pulled by equal number of times. This number is designed in a way such that the event $\mathcal{E}$, defined to be the empirical means of all arms within a small neighborhood of their true means, happens with probability $1-\delta$ (See Definition~\ref{def:all-happen} and Claim~\ref{cla:all-happen}). Note that $\mathcal{E}$ is defined for all rounds and the length of the neighborhood becomes smaller as the algorithm proceeds. We are able to prove that when $\mathcal{E}$ happens, the algorithm returns the desired set of \arm and has small query complexity.

To prove the correctness of the algorithm, we first show that when conditioning on $\mathcal{E}$, the algorithm always accepts a top-$K$ arm in $A$ (Lemma~\ref{lem:accept}) and rejects a non-top-$K$ arm in $B$ (Lemma~\ref{lem:reject}). The key observation here is that our algorithm never introduces any regret due to arms in $A$ and $B$. We then use the key Lemma~\ref{lem:tDelta} to upper bound the regret that may be introduced due to the remaining arms. Once this upper bound is not more than $\epsilon K$ (i.e. the total budget for regret), we can choose the remaining $(K-|A|)$ arms without further samplings. Details about this analysis can be found in Section~\ref{sec:correctness-basic}.

We analyze of the query complexity of our algorithm in Section~\ref{sec:complexity-basic}. We establish data-dependent bound by relating the number of pulls to each arms to both their $\Delta_i$'s and $\Delta_{K-t}$ (Lemma~\ref{lem:r-star} and Lemma~\ref{lem:r-i}).

\subsection{Correctness of Algorithm~\ref{alg:basic}}
\label{sec:correctness-basic}

We first define an event $\mathcal{E}$ which we will condition on in the rest of the analysis.
\begin{definition}
  \label{def:all-happen}
  Let $\mathcal{E}$ be the event that $|\ttheta_i^r - \theta_i| < 2^{-r}$ for all $r \geq 1$ and $i\in S_r$.
\end{definition}

\begin{claim}
  \label{cla:all-happen}
  $\Pr[\mathcal{E}] \geq 1 - \delta$.
\end{claim}

\begin{proof}
By Hoeffding's inequality, we can show that for any fixed $r$ and $i$, $\Pr\left[|\ttheta_i^r - \theta_i| \geq 2^{-r}\right] \leq  2(\frac{\delta}{2nr^2})^2 \leq \frac{\delta}{2nr^2}$. By a union bound,
\[
    \Pr[\neg \mathcal{E}] \leq \sum_{r=1}^{\infty}\sum_{i\in S_r}\Pr \left[|\ttheta_i^r - \theta_i| \geq 2^{-r}\right] \leq     \sum_{r=1}^{\infty}\frac{\delta}{2r^2}
    \leq \delta.
\]
\end{proof}



The following lemma will be a very useful tool for our analysis.

\begin{lemma}
  \label{lem:delta-order}
Given $\mu_1 \geq \ldots \geq \mu_n$ and $\Delta > 0$, assuming that $|\tmu_i - \mu_i| \leq \Delta$ for all $i \in [n]$, and letting $y_1 \geq \ldots \geq y_n$ be the sorted version of $\tmu_1, \ldots, \tmu_n$, we have $|y_i - \mu_i| \leq \Delta~~\text{for all}~ i \in [n]$.
\end{lemma}
\begin{proof}
  Suppose $y_i > \mu_i + \Delta$  for any $i \in [n]$, we must have
  $y_1 \geq \ldots \geq y_i > \mu_i + \Delta$. On the other hand, there can not be more than $i-1$ numbers among $\tmu_1,\ldots, \tmu_n$ (the only candidates are $\tmu_1, \ldots, \tmu_{i-1}$) that are larger than $\mu_i + \Delta$. A contradition. We thus have $y_i \le \mu_i + \Delta$  for all $i \in [n]$.  Similarly, we can show that $y_i \geq \mu_i - \Delta$ for all $i \in [n]$.
\end{proof}

We now prove that conditioned on $\mathcal{E}$, the algorithm always accepts a desired arm in $A$.

\begin{lemma}
  \label{lem:accept}
Conditioned on $\mathcal{E}$, during the run of Algorithm~\ref{alg:basic}, $A \subseteq \{1, 2, \dots, K\}$, that is, all arms in $A$ are among the top-$K$ arms.
\end{lemma}

\begin{proof}
We prove by induction on the round $r$.  The lemma holds trivially when $r = 0$ ($A = \emptyset$).  Now fix a round $r \ge 1$, and let $x$ be the arm that is added to $A$ at Line~\ref{line:accept} of Algorithm~\ref{alg:basic}.  By the induction hypothesis, assuming that before round $r$ all arms in $A$ are in $[K]$, our goal is to show $x \in [K]$.

By the  {\em inner while} condition we have
\begin{equation}
\label{eq:c-1}
\ttheta_{x}^r - \ttheta_a(S_r) > 2\cdot 2^{-r}.
\end{equation}
For any $m \in [K - |A| + 1, |S_r|]$, let $j$ be the arm of the $m$-th largest true-mean in $S_r$, and $j'$ be the arm of the $m$-th largest empirical-mean in $S_r$.  Since $m \geq K - |A| + 1$, we must have $j \not \in [K]$ and $\ttheta_{j'}^r \leq \ttheta_a(S_r)$. By Lemma~\ref{lem:delta-order} we also have $|\ttheta_{j'}^r - \theta_j| < 2^{-r}$. We thus have
  $$\theta_x > \ttheta_x^r - 2^{-r} \overset{\text{by } (\ref{eq:c-1})}{>} \ttheta_{a}(S_r) + 2^{-r} >  \ttheta_{j'}^r + 2^{-r} > \theta_j.$$
  That is,  \emph{at least} $|S_r| - K + |A|$ arms in $S_r$ have  true-means smaller than arm $x$. On the other hand,  $|S_r| - K + |A|$ arms in $S_r$ are not in $[K]$. We therefore conclude that $x$ must be in $[K]$.
\end{proof}

By symmetry, we also have the following lemma, stating that  when $\mathcal{E}$ happens, the algorithm always rejects a non-top-$K$ arm in $B$. We omit the proof because it is almost identical to the proof of Lemma \ref{lem:accept}.
\begin{lemma}
  \label{lem:reject}
Conditioning on $\mathcal{E}$, during the run of Algorithm~\ref{alg:basic}, $B \subseteq \{K+1, K+2,  \dots, n\}$.
\end{lemma}

\begin{lemma}
\label{lem:tDelta}
Conditioned on $\mathcal{E}$, for all rounds $r$ and $i \in S_r$, it holds that
$$\ttheta_i^r - \ttheta_a(S_r) >  \theta_i - \theta_{K+1} - 2\cdot 2^{-r}\ \ \text{and} \ \ \ \ttheta_b(S_r) - \ttheta_i^r >  \theta_K - \theta_i - 2\cdot 2^{-r}.$$
Consequently, we have $\tDelta_i(S_r) \geq \Delta_i - 2\cdot 2^{-r}$ for all rounds $r$ and $i \in S_r$.
\end{lemma}

\begin{proof}
We look at a particular round $r$. Let $j$ be the arm with $(K-|A|+1)$-th largest true-mean in $S_r$. Since by Lemma \ref{lem:accept} we have $A\subseteq [K]$ , it holds that $j \geq K + 1$. By Lemma \ref{lem:delta-order},  we also have $|\ttheta_a(S_r) - \theta_j| < 2^{-r}$. We therefore have for any $i \in S_r$
  \begin{equation}
    \label{eq:Delta-geq}
    \ttheta_i^r - \ttheta_a(S_r) > \theta_i - \theta_j - 2\cdot 2^{-r} \geq \theta_i - \theta_{K+1} - 2\cdot 2^{-r}.
  \end{equation}

  With a similar argument (by symmetry and using Lemma~\ref{lem:reject}), we can show that
    \begin{equation}
    \label{eq:Delta-leq}
    \ttheta_b(S_r) - \ttheta_i^r >  \theta_K - \theta_i - 2\cdot 2^{-r}.
  \end{equation}
Combining (\ref{eq:Delta-geq}), (\ref{eq:Delta-leq}) and the definitions of $\tDelta_i(S_r)$ and $\Delta_i$, the lemma follows.
\end{proof}

Now we are ready to prove the correctness of Theorem~\ref{thm:basic}.  By Lemma~\ref{lem:accept}, all the arms that we add into the set $A$ at Line~\ref{line:accept} are in $[K]$.  The rest of our job is to look at the arms in the set $A'$.

When the algorithm exits the {\em outer while} loop (at round $r = r^*$) and arrives at Line~\ref{line:final}, we have by the condition of the {\em outer while} loop that
\begin{equation}
\label{eq:d-1}
2\cdot 2^{-r^*} \cdot (K - |A|)\leq \eps K.
\end{equation}
Let $m = K - |A|$, and $C = [K]\backslash A = \{i_1, i_2, \ldots, i_m\}$ where $i_1 < i_2 < \ldots < i_m$.  Let $\ttheta_{j_1} \geq \ttheta_{j_2} \geq \ldots \geq \ttheta_{j_m}$ be the $(K-|A|)$ empirical-means of the arms that we pick at Line \ref{line:final}.  Note that it is {\em not} necessary that $j_1 < \ldots < j_m$.  By Lemma \ref{lem:delta-order} and $\mathcal{E}$, for any $s \in [K - |A|]$, we have $|\ttheta_{j_s} - \theta_{i_s}| \leq 2^{-r^*}$ and $|\ttheta_{j_s} - \theta_{j_s}| \leq 2^{-r^*}$.  By the triangle inequality, it holds that
\begin{equation}
\label{eq:d-2}
|\theta_{j_s} - \theta_{i_s}| \le 2 \cdot 2^{-r^*}.
\end{equation}
We thus can bound the error introduced by arms in $A'$ by
$$\sum_{i\in [K]} \theta_i - \sum_{i\in A\cup A'} \theta_i =\sum_{i\in C}\theta_i - \sum_{i\in A'}\theta_i  \overset{by~(\ref{eq:d-2})}{\le} 2\cdot 2^{-r^*} \cdot (K - |A|) \overset{by~(\ref{eq:d-1})}{\le} \eps K.$$

\subsection{Query Complexity of Algorithm~\ref{alg:basic}}
\label{sec:complexity-basic}

Recall (in the statement of Theorem~\ref{thm:basic}) that $t \in \{0, 1, 2, \dots, K - 1\}$ is the largest integer satisfying
\begin{align}
\Delta_{K - t} \cdot t \leq \eps K . \label{eq:e-1}
\end{align}

\begin{lemma}
\label{lem:r-star}
If the algorithm exits the {\em outer while} loop at round $r = r^*$, then we must have
\begin{equation}
  \label{eq:e-2}
8 \cdot 2^{-r^*} \ge \Delta_{K-t} .
\end{equation}
\end{lemma}
\begin{proof}
  We show that once $2^{-r} < \Delta_{K-t} / 4$, the algorithm will exit the outer while loop after executing round $r$. So any valid round $r$ must satisfy $2^{-r} \geq \Delta_{K-t}/8$ and the lemma holds trivially.

To this end, assume now we are in round $r$ and $2^{-r} < \Delta_{K-t}/4$, we have that for any $i \in S_{r}$ and $i \leq K - t$,
\begin{align}
\tDelta_i(S_{r})  \geq \ttheta_i^{r} - \ttheta_{a}(S_{r}) &> \theta_i - \theta_{K+1} - 2\cdot 2^{-{r}} \quad (\text{Lemma \ref{lem:tDelta}}) \nonumber \\
&= \Delta_i - 2 \cdot 2^{-{r}} \nonumber \\
&\ge \Delta_{K-t} - 2\cdot 2^{-{r}} \quad \quad (\mbox{since}~i \leq K-t) \nonumber \\
&> 2\cdot 2^{-{r}}.  \nonumber
\end{align}
Thus the condition of the {\em inner while} loop is satisfied, which means that all arms $i$ with $i \le K - t$ will be added into $A$. Therefore we have $\abs{A} \ge K - t$ when the algorithm exits the inner while loop. We then have
\begin{equation*}
2 \cdot 2^{-r} \cdot (K - |A|) \leq 2 \cdot 2^{-r} \cdot t <\frac{1}{2}\Delta_{K-t} \cdot t \overset{by~(\ref{eq:e-1})}{\leq} \eps K/2 \leq \eps K,
\end{equation*}
so the algorithm exits the outter loop.
\end{proof}

\begin{lemma}
\label{lem:r-i}
For any arm $i$, let $r_i$ be the round where arm $i$ is removed from the candidate set if this ever happens; otherwise set $r_i = r^*$.  We must have
\begin{equation}
\label{eq:e-3}
8 \cdot 2^{-r_i} \ge \Delta_i.
\end{equation}
\end{lemma}

\begin{proof}
Suppose for contradiction that $8 \cdot 2^{-r_i} < \Delta_i$. By Lemma~\ref{lem:tDelta}, we have
\[
\tDelta_i(S_{r_i-1})  \geq \Delta_i - 2 \cdot 2^{-(r_i - 1)} > 8 \cdot 2^{-r_i} - 2 \cdot 2^{-(r_i - 1)} = 2 \cdot 2^{-(r_i - 1)}.
\]
This means that arm $i$ would have been added either to $A$ or $B$ at or before round $(r_i - 1)$, which contradicts to the fact that $i \in S_{r_i}$.
\end{proof}

With Lemma~\ref{lem:r-star} and Lemma~\ref{lem:r-i}, we are ready to analyze the query complexity of the algorithm in Theorem~\ref{thm:basic}.  We can bound the number of pulls on each arm $i$ by at most
\begin{align}
  \sum_{j=1}^{r_i} 2^{2j} \cdot \log (2 n j^2/\delta) &\leq O\left(\log(r_i \cdot n \delta^{-1})\cdot 2^{2r_i}\right). \label{eq:e-4}
\end{align}
Now let us upper-bound the RHS of (\ref{eq:e-4}). First, if $i \in A$, then by (\ref{eq:e-3}) we know that $r_i \le \log_2 \Delta_i^{-1} + O(1)$. Second, by (\ref{eq:e-2}) we have $r_i \le r^* \le \log_2 \Delta_{K-t}^{-1} + O(1)$.
Third, since $2^{-r^*} \ge \eps/2$ (otherwise the algorithm will exit the {\em outer while} loop), we have $r_i \le r^* \le \log_2 \eps^{-1} + O(1)$.  To summarize, we have $r_i \leq \log_2 \min\{\Delta_i^{-1}, \Delta_{K-t}^{-1}, \epsilon^{-1}\} + O(1) = \log_2 \min\{\Delta_i^{-1}, (\Delta_t^{\epsilon})^{-1}\}$ (recall that $\Delta_t^\eps = \max\{\eps, \Delta_{K-t}\}$). We thus can upper-bound the RHS of (\ref{eq:e-4}) by
\begin{equation*}
O\left( (\log \log (\Delta_t^{\eps})^{-1} + \log n + \log \delta^{-1}) \cdot \min\{(\Delta_i)^{-2}, (\Delta_t^\eps)^{-2}\} \right).
\end{equation*}
The total cost is a summation over all $n$ arms.

\section{An Improved Algorithm for $\epsilon$-top-$K$ Arms}
\label{sec:improved}

In this section, we present the improved algorithm for identifying the \arm~and prove that the algorithm succeeds with probability $1-\delta$ with query complexity $O((\log \epsilon^{-1} + \log \delta^{-1})H^{(t, \epsilon)})$ (Theorem~\ref{thm:improved}). This algorithm reduces the $\log n$ factor in the query complexity of Algorithm~\ref{alg:basic} to $\log \epsilon^{-1}$ and is substantially more complex than Algorithm~\ref{alg:basic}.

The main procedure of the improved algorithm is described in Algorithm~\ref{alg:improved}. For this algorithm, we that assume $K\leq n/2$. For the case where $K > n/2$, we can apply the same algorithm to identify the $\eps$-bottom-$(n-K)$ arms and report the rest arms to be the \arm. Similarly to Algorithm~\ref{alg:basic}, the improved algorithm also goes by rounds and keeps a set $A$ of accepted arms, a set $B$ of rejected arms, and a set $S$ of undecided arms. However, we can no longer guarantee that all the arms accepted in $A$ and rejected in $B$ are correctly classified -- otherwise, we need to apply a union bound over all arms and this would incur an extra $\log n$ factor. To solve this problem, we have to allow a few number of mistakes. We now illustrate the high-level idea as follows.

Given a set of $n$ arms $\{\theta_1 \geq \theta_2 \geq \dots \geq \theta_n\}$, if we pull every arm $c \cdot \Delta_{.8n}^{-2} (\log \epsilon^{-1} + \log \delta^{-1})$ times for some large enough constant $c$, and discard the $.1 n$ arms with the lowest empirical means, it can be shown by standard probabilistic method that at most $\epsilon^2 K$ top-$K$ arms may be mistakenly discarded with probability $1-\delta$. Note that the constants $.8$ and $.1$ are arbitrary as long as $K/n < .8 < 1 - .1$. This procedure is described in Algorithm~\ref{algo:elim} and analyzed in Lemma~\ref{lem:elim}. Similarly, if $.2n < K$ and we pull every arm $c \cdot \Delta_{.2n}^{-2}(\log \epsilon^{-1} + \log \delta^{-1})$ times for some large enough constant $c$, and accept the $.1 n$ arms with the highest empirical means, with probability $1 - \delta$,  at most $\epsilon^2 K$ non-top-$K$ arms may be mistakenly accepted. This procedure is described in Algorithm~\ref{algo:reverse-elim} and analyzed in Lemma~\ref{lem:reverse-elim}. Algorithm~\ref{alg:improved} uses these two subroutines to repeatedly accept and reject arms, and makes sure that with high probability, the total number of mistakenly accepted or rejected arms is at most $O(\epsilon^2 K)$ (Lemma~\ref{lem:misclassify}). These mistakes lead to $O(\epsilon^2 K)$ total regret -- negligible when compared to our $\epsilon K$ budget. In this way, the improved algorithm keeps accepting and rejecting arms as Algorithm~\ref{alg:basic} does, while introducing negligible regret (while Algorithm~\ref{alg:basic} introduces none). The termination condition is also similar to Algorithm~\ref{alg:basic} in Line~\ref{line:b-4} of Algorithm~\ref{alg:improved} so that the query complexity is related to $H^{(t, \epsilon)}$ rather than $H^{(0, \epsilon)}$.

However, there is an extra termination condition and many extra efforts in the improved algorithm because of the few allowed mistakes. For our adaptive algorithm, in order to estimate $\Delta_{.8n}$ and $\Delta_{.2n}$ (and other gaps as the algorithm proceeds), we need to estimate $\theta_K$, $\theta_{.8n}$ and $\theta_{.2 n}$ with $O(\phi^{-2} (\log \epsilon^{-1} + \log \delta^{-1})$ pulls, where $\phi^{-2} = \Omega(\min\{\Delta_{.8n}, \Delta_{.2n}\})$. However, using these many pulls, we can only estimate the mean of an arm that is close to the target index, rather than with the exact index. This procedure is presented in Section~\ref{sec:top-k} and Algorithm~\ref{alg:k-th}. We use this subroutine to estimate $\theta_{.8n}$ as $\theta^+$, $\theta_{.2 n}$ as $\theta^-$ in Algorithm~\ref{alg:improved}, and use two estimations $\theta_K^+$ and $\theta_K^-$ to sandwich $\theta_K$. (The precise statement can be found in Lemma~\ref{lem:sandwich}.) When $\theta_K^+$ and $\theta_K^-$ are close to each other, we can use $\theta^- - \theta_K^{-}$ and $\theta_K^+ - \theta^+$ as estimations of $\Delta_{.2n}$ and $\Delta_{.8n}$; otherwise, it means that there is a big gap in the neighborhood of the $K$-th arm, and we can easily separate the top-$K$ arms from the rest using the subprocedure {\sc EpsSplit} described in Lemma~\ref{lem:gap} and quit the procedure (in Line~\ref{line:b-3} of the algorithm).

We now dive into the details of the improved algorithm. We start by introducing the useful subroutines.

\subsection{Estimating the $K$-th Largest Arm}  
\label{sec:top-k}

\begin{algorithm}[t]
\KwIn{$S$: set of arms; $K$: top-$K$; $\tau$: an relative error; $\delta$: error probability; $\phi$: an additive error}
\KwOut{an arm  whose true-mean is close to the $K$-th largest true-mean}
\DontPrintSemicolon
\caption{{\sc EstKthArm}$(S, K, \tau, \phi, \delta)$ \label{alg:k-th}}
set $R_1 \gets S, r \gets 1$\;
set $\tau_1 \gets \frac{\tau}{4}, \phi_1 \gets \frac{\phi}{4}, \delta_1 \gets \frac{\delta}{8}$\;
\While{$|R_r| > K$}{
  for each $i\in R_r$, pull $\frac{8}{\phi_r^2}\ln(\frac{1}{\tau_r\delta_r\delta})$ times; let $\ttheta_i^r$ be its empirical-mean \label{line:pull}\;
  let $R_{r+1}$ be the set of $\max\{K, \lceil {|R_r|}/{2} \rceil \}$ arms that have the largest empirical-means among $R_r$\;
  set $\tau_{r+1} \gets 3\tau_{r}/4, \phi_{r+1} \gets 3\phi_r/4, \delta_{r+1} \gets \delta_r/2$\;
  $r \gets r+1$\;
}
set $r^* \gets r$\;
set $\tp_1 \geq \tp_2 \geq \ldots \geq \tp_{|R_{r^*}|}$ be the sorted version of $\{\ttheta_i^{r^*} ~|~ i \in R_{r^*}\}$\;
uniformly sample an arm from $\{ i \in R_{r^*} ~|~  \ttheta_i^{r^*} \leq \tp_{(1-\tau/2)K}\}$ and output it \label{line:sample} \;
\end{algorithm}

In the subsection we present an algorithm that try to find an arm whose true-mean is close to the $K$-th largest true-mean, which will be used as a subroutine in our improved algorithm for \arm.

\begin{theorem}
  \label{thm:estK}
For a set of arms $S = \{\theta_1 \ge \ldots \ge \theta_{\abs{S}}\}$,
there is an algorithm, denoted by {\sc EstKthArm}$(S, K, \tau, \phi, \delta)$, that outputs an arm $i$ such that $\theta_i \in [\theta_K - \phi, \theta_{(1-\tau)K} + \phi]$ with probability at least $1 - \delta$, using $O\left(\frac{|S|}{\phi^2} \cdot (\log\tau^{-1}+ \log\delta^{-1}) \right)$ pulls in total.
\end{theorem}

We described the algorithm in Algorithm~\ref{alg:k-th}.  In the high level, the algorithm works in rounds, and in each round it tries to find the top half arms in the current set, and discard the rest.  We continue until there are at most $K$ arms left, and then we choose the output arm randomly from those with the lowest empirical-means in the remaining arms.  We are going to prove the following theorem.

\subsubsection{Correctness of Algorithm~\ref{alg:k-th}}
The following lemma is the key to the proof of correctness.

\begin{lemma}
\label{lem:few-remain}

With probability at least $1-\delta/4$, we have that
  $$|\{i \in R_{r^*} ~|~ \theta_i < \theta_K - \phi\}| \le \tau \delta K/4.$$
\end{lemma}

\begin{proof}
We first define a few notations.
\begin{itemize}
\item $H_r = \{i \in S ~|~ \theta_i \geq \theta_K - \sum_{\ell \in [r]}\phi_\ell\}$.

\item $L_r = S \backslash H_r$.

\item $k_r =  (1 - \sum_{\ell \in [r]}\tau_\ell \delta/4)K$.

\item For any $i \in [n]$ and round $r$, let $X_i^r = \mathbf{1}\{|\ttheta_i^r - \theta_i| \ge \phi_r/2\}$ where $\ttheta_i^r$ is the empirical-mean of arm $i$ at round $r$ (after been pulled by $\frac{8}{\phi_r^2}\ln(\frac{1}{\tau_r\delta_r\delta})$ times at Line~\ref{line:pull}).

\item $R'_r \subseteq R_r$: the top $k_{r-1}$ arms in $R_r$ with the largest true-means.

\item $A_r = \{i \in R'_r\ |\ X_i = 0\}$.  $A_r \subseteq R'_r \subseteq R_r$.

\item $C_r = \{i \in L_r \cap R_r\ |\ X_i = 1\}$.  $C_r \subseteq R_r$.

\end{itemize}

We define the following event.  Intuitively, it tells that most of the arms we put in $R_r$ for the next round processing fall into the set $H_r$ of high true-means.
\begin{equation*}
\label{eq:f-1}
\mathcal{E}_r(r \ge 2):~ |H_{r-1}\cap R_r| \le k_{r-1},
\end{equation*}
and we define $\mathcal{E}_1$ to be an always true event.
We will prove by induction the following inequality.
 \begin{equation}
    \label{eq:f-2}
    \Pr[\mathcal{E}_{r+1}\ |\ \neg \mathcal{E}_r] \le \delta_r ~~\text{for each}~r \geq 1.
  \end{equation}

We focus a particular round $r \ge 1$. Define event $\mathcal{E}_A: \abs{A_r} \le k_r$, and event $\mathcal{E}_C: \abs{C_r} \le \abs{R_{r+1}} - k_r$.

\begin{claim}
\label{cla:event-A}

$\Pr[\mathcal{E}_A\ |\ \neg \mathcal{E}_r] \le \delta_r/4$.
\end{claim}

\begin{proof}
By a Hoeffding's inequality, we have for each $i \in R_r$, we have $\Pr[X_i^r = 1] \le \tau_r \delta_r \delta / 16$.  We bound the probability the $\mathcal{E}_A$ happens by a Markov's inequality.
\begin{align*}
    \Pr[\mathcal{E}_A\ |\ \neg \mathcal{E}_r] = \Pr[|A_r| \le k_r\ |\ \neg \mathcal{E}_r]
                  =& \Pr[|R_r'| - |A_r| \ge k_{r-1} - k_r\ |\ \neg \mathcal{E}_r]\\
                  \le& \frac{\bE\left[\sum_{i\in R_r'}X_i \ |\ \neg \mathcal{E}_r] \right]}{k_{r-1} - k_r} = \frac{\bE\left[\sum_{i\in R_r'}X_i \right]}{k_{r-1} - k_r}\\
                  \le& \frac{\tau_r\delta_r\delta k_{r-1}/16}{\tau_r\delta K/4} \le \frac{\tau_r\delta_r\delta K/16}{\tau_r\delta K/4}\\
                  =&  \delta_r/4.
\end{align*}
\end{proof}

\begin{claim}
\label{cla:event-C}
$\Pr[\mathcal{E}_C \ |\ \neg \mathcal{E}_A, \neg \mathcal{E}_r] \le 3\delta_r/4$.
\end{claim}

\begin{proof}
\begin{align}
\Pr[\mathcal{E}_C \ |\ \neg \mathcal{E}_A, \neg \mathcal{E}_r] =& \Pr[\abs{C_r} \le \abs{R_{r+1}} - k_r \ |\ \neg \mathcal{E}_A, \neg \mathcal{E}_r]  \nonumber \\
\le& \frac{\bE[\abs{C_r}\ |\ \neg \mathcal{E}_A, \neg \mathcal{E}_r]}{\abs{R_{r+1}} - k_r} \quad \quad \quad \text{(Markov's inequality)} \label{eq:g-1} \\
\le& \frac{(\abs{R_r} - k_r) \cdot \tau_r\delta \delta_r / 16}{\abs{R_{r+1}} - k_r} \label{eq:g-2} \\
\le& \frac{(2\abs{R_{r+1}} - k_r) \cdot \tau_r\delta \delta_r / 16}{\abs{R_{r+1}} - k_r} \label{eq:g-3} \\
\le& 3\delta_r/4,  \label{eq:g-4}
\end{align}
where (\ref{eq:g-1}) to (\ref{eq:g-2}) is due to the fact that conditioned on $\neg \mathcal{E}_r$, we have
\[
\abs{C_r} \le \abs{L_r \cap R_r} \le \abs{L_{r-1} \cap R_r} = \abs{R_r} - \abs{H_{r-1} \cap R_r} \overset{\neg \mathcal{E}_r}{\le} \abs{R_r} - k_{r-1} \le \abs{R_r} - k_r.
\]
And (\ref{eq:g-3}) to (\ref{eq:g-4}) is due to the following:  If  $\abs{R_{r+1}} \ge 2K$, then since $k_r \le K$ we have $ \frac{2\abs{R_{r+1}} - k_r}{\abs{R_{r+1}} - k_r} \le 3$.  Otherwise if $\abs{R_{r+1}} < 2K$, by the definition of $k_r$, we have
$$\abs{R_{r+1}} - k_r > K - k_r = \sum_{i \in [r-1]}\tau_i\delta/4 \ge \tau_r\delta K/3.$$
On the other hand, we have $2\abs{R_{r+1}} - k_r \le 4K$.  We thus have (\ref{eq:g-4}) $\le 3\delta_r/4$.
\end{proof}

\begin{claim}
\label{cla:good-next}
Conditioned on $\mathcal{E}_C, \neg \mathcal{E}_A, \neg \mathcal{E}_r$, we have $\abs{H_r \cap R_{r+1}} \ge k_r$, or, $\Pr[\mathcal{E}_{r+1}\ |\ \mathcal{E}_C, \neg \mathcal{E}_A, \neg \mathcal{E}_r] = 0$.
\end{claim}
\begin{proof}
First, conditioned on $\neg \mathcal{E}_r$, we have
\begin{equation}
\label{eq:f-3}
A_r \subseteq R'_r \subseteq H_{r-1}.
\end{equation}

We prove the claim by analyzing two cases.
\begin{enumerate}
\item $((L_r \cap R_r) \backslash C_r) \cap R_{r+1} = \emptyset$.  We thus have $(L_r \cap R_{r+1}) \subseteq C_r$, which implies
$$\abs{H_r \cap R_{r+1}} \ge \abs{R_{r+1}} - \abs{L_r \cap R_{r+1}} \ge \abs{R_{r+1}} - \abs{C_r} \overset{\mathcal{E}_C}{\ge} k_r.$$

\item $((L_r \cap R_r) \backslash C_r) \cap R_{r+1} \neq \emptyset$.  We can show that
\begin{equation}
\label{eq:g-6}
A_r \subseteq R_{r+1}
\end{equation}
Indeed, for any $j \in A_r (\subseteq H_{r-1}\ by\ (\ref{eq:f-3}))$ and any $i \in (L_r \cap R_r) \backslash C_r$, we have $\ttheta_j > \ttheta_i$, since $\theta_j \ge \theta_K - \sum_{\ell \in [r-1]} \phi_\ell$, while $\theta_i < \theta_K - \sum_{\ell \in [r]} \phi_\ell$ (by the definition of $L_r$).  Thus conditioned on $\neg \mathcal{E}_A$, (\ref{eq:f-3}) and (\ref{eq:g-6}) we have
$$\abs{H_r \cap R_{r+1}} \ge \abs{H_{r-1} \cap R_{r+1}} \ge \abs{A_r} \ge k_r.$$
\end{enumerate}
\end{proof}

Now we try to prove (\ref{eq:f-2}).
\begin{align*}
    \Pr[\mathcal{E}_{r+1}\ |\ \neg \mathcal{E}_r] =& \Pr[\mathcal{E}_A\ |\ \neg \mathcal{E}_r] + \Pr[\mathcal{E}_C\ |\ \neg \mathcal{E}_A, \neg \mathcal{E}_r] + \Pr[\mathcal{E}_{r+1}\ |\ \neg \mathcal{E}_A, \neg \mathcal{E}_r]  \\
    \le&  \delta/4 + 3\delta/4 + 0 \quad \quad \text{(Claim~\ref{cla:event-A}, Claim~\ref{cla:event-C} and Claim~\ref{cla:good-next})} \\
    =& \delta.
\end{align*}
Using (\ref{eq:f-2}) and summing over all $\ell \in [r-1]$, we have $\Pr[\mathcal{E}_r] \le \sum_{\ell \in [r-1]} \delta_i < \delta/4$ for any $r \ge 2$. In other words, with probability at least $1 - \delta/4$, we have for any $r \le r^*$, $\abs{H_{r-1} \cap R_r} \ge k_{r-1}$, or $\abs{L_{r-1} \cap R_r} \le K - k_{r-1} \le 1 - (1 - \tau\delta K /4) = \tau\delta K /4$, which gives the lemma.
\end{proof}

Now we are ready to prove the correctness of Theorem~\ref{thm:estK}.  Let
\begin{itemize}
\item $P = \{i \in R_{r^*} ~|~ \ttheta_i^{r^*} > \tp_{(1-\tau/2)K}\}$.

\item $Q = R_{r^*}\backslash P$.

\item $L^* = \{i \in R_{r^*} ~|~ \theta_i < \theta_K - \phi\}$.

\item For each $i \in R_{r^*}$, let $Y_i = \mathbf{1}\{|\ttheta_i^{r^*} - \theta_i| > \phi/2\}$.
\end{itemize}

\begin{claim}
\label{cla:lb}
$\Pr[\theta_a \ge \theta_K - \phi] \ge 1 - 3\delta/4$.
\end{claim}
\begin{proof}
At Line~\ref{line:sample} of Algorithm~\ref{alg:k-th}, we randomly sampled an arm $a$ from $Q$.  Since $\abs{Q} \ge \tau K$ and $\Pr[\abs{L^*} < \tau\delta K/4] > 1 - \delta/4$ (by Lemma \ref{lem:few-remain}),  we have
  \begin{equation}
    \label{eq:h-2}
    \Pr[a \in L^*] \le \Pr[\abs{L^*} \geq \tau\delta K/4] + \Pr[a \in L^* \ |\ \abs{L^*} < \tau\delta K/4] \leq \delta/4 + \frac{\tau\delta K/4}{\tau K/2} = 3\delta/4.
  \end{equation}
\end{proof}

\begin{claim}
\label{cla:ub}
$\Pr[\theta_a \le \theta_{(1-\tau)K} + \phi] \ge 1 - \delta/4$.
\end{claim}
\begin{proof}
We first show that if $Y_a = 0$ and $\sum_{i \in P} Y_i \le \tau K/2$, then $\theta_a \le \theta_{(1-\tau)K} + \phi$.   Let $P_0 = \{i \in P \ |\ Y_i = 0\}$. By definition of $P$ and the assumption  that $\sum_{i \in P} Y_i \leq \tau K/2$, we have $\abs{P_0} \geq \abs{P} - \tau K/2 \geq (1 - \tau)K$.
Let $b$ be the arm in $P_0$ that has the minimum true-mean, then we must have
\begin{equation}
\label{eq:i-1}
\theta_b \le \theta_{\abs{P_0}} \le \theta_{(1-\tau)K}.
\end{equation}
  Since $a \in Q$ and $b \in P$, we have $\ttheta_a \le \ttheta_b$, which, together with the facts that $Y_a = Y_b = 0$ and (\ref{eq:i-1}), gives
$$\theta_a \le \theta_b + \phi \le \theta_{(1-\tau)K} + \phi.$$

We now bound the probabilities that the two conditions hold.  By a Hoeffding's inequality, we have $\Pr[Y_i = 1] \le \tau\delta/16$ for any $i \in R_{r^*}$.  By a Markov's inequality and the fact that $P \subseteq R_{r^*}$ (by definition), we have
\begin{equation*}
\Pr\left[\sum_{i\in P}Y_i > \tau K/2\right] \le \Pr\left[\sum_{i\in R_{r^*}}Y_i > \tau K/2\right] \le \frac{\bE \left[\sum_{i\in R_{r^*}}Y_i \right]}{\tau K/2} \le \frac{\tau\delta K/16}{\tau K/2} = \delta/8.
\end{equation*}
We thus have $\Pr\left[(Y_a = 0) \wedge (\sum_{i \in P} Y_i \le \tau K/2)\right] \ge 1 - \delta/16 - \delta/8 \ge 1 - \delta/4$.

The correctness of Theorem~\ref{thm:estK} immediately follows from Claim~\ref{cla:ub} and Claim~\ref{cla:lb}.
\end{proof}

\subsubsection{Complexity Algorithm~\ref{alg:k-th}}
We can bound the total number of pulls of Algorithm~\ref{alg:k-th} by simply summing up the number of pulls at each round.

\begin{align*}
    O\left(\sum_{r=1}^{\log{|S|}/{K}} \frac{|R_r|}{\phi_r^2} \log\frac{1}{\tau_r\delta_r\delta}\right) =& O\left(\sum_{r=1}^{\infty} \frac{2^{-r}|S|}{(3/4)^{2r}\phi^2}\cdot \left(\log \frac{1}{(3/4)^r\tau} + \log\frac{1}{(1/2)^r\delta} + \log\frac{1}{\delta}\right) \right)\\
=& O\left(\frac{|S|}{\phi^2}\sum_{r=1}^{\infty}\left(\frac{8}{9}\right)^r\cdot \left(r + \log\frac{1}{\tau} + \log\frac{1}{\delta}\right) \right)\\
=& O\left(\frac{|S|}{\phi^2} \cdot\left(\log \frac{1}{\tau} + \log\frac{1}{\delta}\right) \right).
\end{align*}
The last equality follows from the fact that $\sum_{r=1}^{\infty}(8/9)^r\cdot r = O(1)$.

Lemma~\ref{lem:few-remain} also implies the following lemma.

\begin{lemma}
\label{lem:gap}

For a set of arms $S = \{\theta_1 \ge \ldots \ge \theta_{\abs{S}} \}$ such that $\theta_{(1-\tau)K} - \theta_{(1+\tau)K +1} \ge \phi$,  there is an algorithm, denoted by {\sc EpsSplit}$(S, K, \tau, \phi, \delta)$, that computes $(2\tau)$-top-$K$ correctly with probability at least $1 - \delta$, using $O\left(\frac{|S|}{\phi^2} \cdot (\log\tau^{-1}+ \log\delta^{-1}) \right)$ pulls in total.
\end{lemma}

\begin{proof}
To prove the lemma, we just need to replace ``$K$'' in Algorithm~\ref{alg:k-th} to be ``$(1-\tau)K$''.  By Lemma~\ref{lem:few-remain} we have with probability $1 - \delta$ that
\begin{equation*}
|\{i \in R_{r^*} ~|~ \theta_i < \theta_{(1-\tau)K} - \phi\}| \le \tau \delta (1-\tau)K \le \tau \delta K.
\end{equation*}
Since $\theta_{(1-\tau)K} - \theta_{(1+\tau)K+1} \ge \phi$ we have
\begin{equation*}
|\{i \in R_{r^*} ~|~ \theta_i < \theta_{(1+\tau)K+1}\}| \le \tau \delta K.
\end{equation*}
Consequently,
\begin{equation}
\label{eq:j-1}
|\{i \in R_{r^*} ~|~ \theta_i < \theta_{K}\}| \le  \tau K + \tau \delta K \le 2\tau K.
\end{equation}
Therefore we can just choose all arms in $R_{r^*}$, together with $K - \abs{R_{r^*}}$ arbitrary arms in the rest $n - \abs{R_{r^*}}$ arms.  By (\ref{eq:j-1}) the total average error is bounded by $2\tau K / K = 2\tau$.
\end{proof}

\subsection{The Improved Algorithm}

In this section,  we introduce an improved algorithm that removes the $\log(n)$-factor in the sample complexity.

We first introduce a few more subroutines (Lemma~\ref{lem:optimal}, Lemma~\ref{lem:elim}, and Lemma~\ref{lem:reverse-elim}) that will be useful for our improved algorithm.

\begin{lemma}\citep{Zhou:14}
\label{lem:optimal}
For a set of arms $S = \{\theta_1 \ge \ldots \ge \theta_{\abs{S}}\}$, there is an algorithm, denoted by {\sc OptMAI}$(S, K, \eps, \delta)$, that computes \arm\ with probability $1 - \delta$, with $O\left(\frac{\abs{S}}{\eps^2} \log(1/\delta)\right)$ pulls.
\end{lemma}

\begin{algorithm}[t]
\KwIn{$S$: set of arms; $K$: top-$K$; $\gamma$: fraction of arms; $\delta$: error probability; $\phi$: an additive error}
\KwOut{Set of arms $T$ with $|T|=\frac{|S|}{10}$ such that at most $\gamma K$ arms in $T$ are in the top-$K$ arms in $S$}
\DontPrintSemicolon
\caption{{\sc Elim}$(S, K, \gamma, \phi, \delta)$ \label{alg:elim}}

For each arm $i$ in $S$, pull $\frac{c}{\phi^2} \cdot (\log\gamma^{-1}+ \log\delta^{-1})$ for some large enough constant $c$. Let $\ttheta_i$ be the empirical-mean of the $i$-th arm.

\Return{$|T|$ arms in $S$ with the smallest empirical-means.}
\label{algo:elim}
\end{algorithm}

\begin{algorithm}[t]
\KwIn{$S$: set of arms; $K$: top-$K$; $\gamma$: fraction of arms; $\delta$: error probability; $\phi$: an additive error}
\KwOut{Set of arms $T$ with $|T|=\frac{|S|}{10}$ such that at most $\gamma K$ arms in $T$ are in the top-$(|S|-K)$ arms in $S$}
\DontPrintSemicolon
\caption{{\sc ReverseElim}$(S, K, \gamma, \phi, \delta)$}

For each arm $i$ in $S$, pull $\frac{c}{\phi^2} \cdot (\log\gamma^{-1}+ \log\delta^{-1})$ for some large enough constant $c$. Let $\ttheta_i$ be the empirical-mean of the $i$-th arm.

\Return{$|T|$ arms in $S$ with the largest empirical-means.}
\label{algo:reverse-elim}
\end{algorithm}

The following two lemmas show how to find a constant fraction of arms in the set of top-$K$ arms and a constant fraction of arms outside the set of top-$K$ arms respectively.

\begin{lemma}
\label{lem:elim}

For a set of arms $S = \{\theta_1 \ge \ldots \ge \theta_{\abs{S}}\}$ such that $\theta_K - \theta_{\frac{\abs{S}+K}{2}} \ge \phi$, $K \le \frac{2}{3}\abs{S}$, there is an algorithm, denoted by
{\sc Elim}$(S, K, \gamma, \phi, \delta)$ in Algorithm~\ref{algo:elim}, that computes $T \subseteq S, \abs{T} = \frac{\abs{S}}{10}$ successfully with probability $1 - \delta$ using $O\left(\frac{|S|}{\phi^2} \cdot (\log\gamma^{-1}+ \log\delta^{-1}) \right)$ pulls in total, such that at most $\gamma K$ arms in $T$ are in the top-$K$ arms in $S$.
\end{lemma}

\begin{proof}
Let $\ttheta_i$ be the empirical-mean of arm $i$ after being pulled by $c \cdot \frac{1}{\phi^2} \log \frac{1}{\gamma \delta}$ times for a sufficiently large constant $c$. By a Hoeffding's inequality, we have
$\Pr\left[|\ttheta_i - \theta_i| \ge {\phi}/{2}\right] \le \gamma \delta/2$.  Let $X_i = \mathbf{1}\{|\ttheta_i - \theta_i| \ge {\phi}/{2}\}$, and thus $\bE[X_i] \le \gamma \delta/2$.  Let $X = \sum_{i \in [K]} X_i$; we  have $\bE[X] \le \gamma \delta K/2$. By a Markov's inequality, we have that with probability at least $1 - \delta/2$, $X \le \gamma K$.  Consequently, with probability at least $1 - \delta/2$, there are at most $\gamma K$ arms $i \in [K]$ with $\ttheta_i \le \theta_i - \phi/2 \le \theta_K - \phi/2$.

Let $L = \{\frac{\abs{S}+K}{2}+1, \ldots, \abs{S}\}$. Since $K \le \frac{2}{3} \abs{S}$, we have $\abs{L} \ge \frac{\abs{S}}{6}$. Using similar argument we can show that with probability $1 - \delta/2$, there are at least $\frac{\abs{S}}{10}$ arms $i \in L$ with $\ttheta_i \le \theta_i + \phi/2 \le \theta_{\frac{\abs{S}+K}{2}} + {\phi}/{2} \le \theta_K - \phi/2$ (since $\theta_K - \theta_{\frac{\abs{S}+K}{2}} \ge \phi$).

Therefore, if we choose $T$ to be the $\frac{\abs{S}}{10}$ arms with the smallest empirical-means, then with probability at least $1 - \delta$, at most $\gamma K$ arms in $T$ are in the top-$K$ arms in $S$.
\end{proof}

\begin{lemma}
\label{lem:reverse-elim}

For a set of arms $S = \{\theta_1 \ge \ldots \ge \theta_{\abs{S}}\}$ such that $\theta_{\frac{K}{2}} - \theta_K \ge \phi$, $K \ge \frac{\abs{S}}{3}$, there is an algorithm, denoted by {\sc ReverseElim}$(S, K, \gamma, \phi, \delta)$ in Algorithm~\ref{algo:reverse-elim}, that computes  $T \subseteq S, \abs{T} = \frac{\abs{S}}{10}$ successfully with probability $1 - \delta$ using $O\left(\frac{|S|}{\phi^2} \cdot (\log\gamma^{-1}+ \log\delta^{-1}) \right)$ pulls in total, such that at most $\gamma K$ arms in $T$ are in the bottom-$(\abs{S}-K)$ arms in $S$.
\end{lemma}

By symmetry, the proof to Lemma~\ref{lem:reverse-elim} is basically the same as that for Lemma~\ref{lem:elim},

\begin{algorithm}[t]
\KwIn{$n$: number of arms; $K$ and $\eps$: see the definition of \arm; $\delta$: error probability. Assume $K \leq n/2$}
\KwOut{\arm}
\DontPrintSemicolon
\caption{{\sc ImprovedTopK}$(n, K, \eps, \delta)$ \label{alg:improved}}
set $S \gets [n]$, $r \gets 1$, $r_\phi \gets 1$\;
set $A, B \gets \emptyset$\;
set $K_L \gets (1-\eps^2)K, K_R \gets (1+\eps^2)K + 1$\;
\While{$S \neq \emptyset$}{
	$r_\phi \gets r_\phi - 1$\;
	\Repeat{$(10(K - \abs{A}) \phi < K\eps)$\ \  or\ \
	$(\theta_K^- - \theta_K^+ > 3\phi)$\ \  or\ \
	$\left(K - \abs{A} \le \frac{\abs{S}}{2} \right) \wedge (\theta_K^+ - \theta^+ > 3\phi)$\ \  or\ \
	$\left(K - \abs{A} > \frac{\abs{S}}{2}\right) \wedge (\theta_K^+ - \theta^+ > 3\phi) \wedge (\theta^- - \theta_K^-  > 3\phi)$ \label{line:a-0}}
	{
		$r_\phi \gets r_\phi + 1$;
		$\phi \gets 2^{-r_\phi}$\;
		$\theta_K^+ \gets \text{\sc EstKthArm}\left(S, K_R - \abs{A}, \frac{\eps^2}{100 r^2}, \phi, \frac{\delta}{100 (r + r_\phi)^2}\right)$; $\theta_K^- \gets \text{\sc EstKthArm}\left(S, K_L - \abs{A}, \frac{\eps^2}{100 r^2}, \phi, \frac{\delta}{100 (r+r_\phi)^2}\right)$ \label{line:a-1} \;
		$\theta^+ \gets \text{\sc EstKthArm}\left(S, \frac{\abs{S} + K - \abs{A}}{2}, \frac{\eps^2}{100 r^2}, \phi, \frac{\delta}{100 (r+r_\phi)^2}\right)$; $\theta^- \gets \text{\sc EstKthArm}\left(S, \frac{K - \abs{A}}{2}, \frac{\eps^2}{100 r^2}, \phi, \frac{\delta}{100 (r+r_\phi)^2}\right)$ \label{line:a-2}
  	}
  	
  	\If{$(10(K - \abs{A}) \phi < K\eps)$ \label{line:b-1}}{\Return{{\sc OptMAI}$\left(S, K - \abs{A}, \phi, \frac{\delta}{100} \right) \cup A$}\label{line:b-4}}
  	\If{$(\theta_K^- - \theta_K^+ > 3\phi)$ \label{line:b-2}}{\Return{{\sc EpsSplit}$\left(S, K - \abs{A}, \frac{K_R - K_L}{K - \abs{A}}, \phi, \frac{\delta}{100} \right) \cup A$} \label{line:b-3}}
  	
  	$U \gets \text{\sc Elim}\left(S, K - \abs{A}, \frac{\eps^2}{100 r^2}, \phi, \frac{\delta}{100 r^2}\right)$ \label{line:c-1}\;
  	\If{$(K - \abs{A}) > \frac{\abs{S}}{2}$ \label{line:c-2}}{$V \gets \text{\sc ReverseElim}\left(S, K - \abs{A}, \frac{\eps^2}{100 r^2}, \phi, \frac{\delta}{100 r^2}\right)$\label{line:c-3}}
  	$r \gets r+1$; $S \gets S \backslash (U \cup V)$; $A \gets A \cup V$; $B \gets B \cup U$ \label{line:c-4}\;
}
\Return $A$.
\end{algorithm}

\medskip

Now we are ready to show our main result.

\begin{theorem}
\label{thm:improved}
Algorithm \ref{alg:improved} computes \arm\ with probability at least $1 - \delta$, and pulls the arms at most
\begin{equation}
O\left(\sum_{j \in [n]} \min\left\{(\Delta_j)^{-2}, (\Psi_t^\eps)^{-2}\right\} \left(\log \frac{1}{\eps} + \log  \frac{1}{\delta} \right) \right)
\end{equation}
times, where $t$ and $\Psi_t^{\eps}$ are defined in \eqref{eq:def-t-eps-K} and \eqref{eq:def-Psi-t-eps} respectively.
\end{theorem}
It is worthwhile to note that the proposed algorithm is mainly for the theoretical interest and is rather complicated in terms of implementation. Thus, we omit the empirical study of this algorithm in the experimental section. 

In the rest of this section we prove Theorem~\ref{thm:improved} by showing the correctness of Algorithm~\ref{alg:improved} and the analyzing its query complexity.

\subsubsection{Correctness of Algorithm~\ref{alg:improved}}
\label{sec:correctness-improved}

Define $\mathcal{E}_1$ to be the event that all calls to the subroutine {\sc EstKthArm} succeed.

\begin{claim}
\label{cla:selection}
$\Pr[\mathcal{E}_1] \ge 1 - \delta/10$.
\end{claim}

\begin{proof}
Note that $(r+r_\phi)$ increases every time we call the four {\sc EstKthArm}'s at Line~\ref{line:a-1} and Line~\ref{line:a-2}. Therefore by Theorem~\ref{thm:estK} we can bound the error probability of all calls to {\sc EstKthArm} by
$ 4 \cdot \sum_{z=1}^\infty \frac{\delta}{100 z^2} \le \frac{\delta}{10}.
$
\end{proof}

Define $\mathcal{E}_2$ to be the event that all calls to the subroutines {\sc Elim} and {\sc ReverseElim} succeed.  Since $r$ increases every time we call the two subroutines, by similar arguments we have:

\begin{claim}
\label{cla:selection}
$\Pr[\mathcal{E}_2] \ge 1 - \delta/20$.
\end{claim}

Define $\mathcal{E} = \mathcal{E}_1 \cup \mathcal{E}_2$; we thus have $\Pr[\mathcal{E}] \ge 1 - \delta/5$.

\medskip

We next show that the misclassified arms are negligible during the run of the algorithm.

\begin{lemma}
\label{lem:misclassify}
Conditioned on $\mathcal{E}$, suppose that the conditions of Lemma~\ref{lem:elim} and Lemma~\ref{lem:reverse-elim} always hold during the run of the Algorithm~\ref{alg:improved}, then we always have
\begin{enumerate}
\item The number of non-top-$K$ arms in $A$, denoted by $\iota_A$, is no more than $\frac{\eps^2 K}{40}$.

\item The number of top-$K$ arms in $B$, denoted by $\iota_B$, is no more than  $\frac{\eps^2 K}{40}$.
\end{enumerate}
\end{lemma}

\begin{proof}
By Lemma~\ref{lem:elim} we have $\iota_A \le \sum_{r=1}^\infty \frac{\eps^2}{100 r^2} \cdot K \le \frac{\eps^2 K}{40}$. Similarly, by Lemma \ref{lem:reverse-elim} we have $\iota_B \le \sum_{r=1}^\infty \frac{\eps^2}{100 r^2} \cdot K \le \frac{\eps^2 K}{40}$.
\end{proof}

We now show that the conditions of Lemma~\ref{lem:elim} and Lemma~\ref{lem:reverse-elim} do hold.  We first introducing a lemma showing that $\theta_K$ is sandwiched by $\theta_K^+$ and $\theta_K^-$   during the run of Algorithm~\ref{alg:improved}.

\begin{lemma}
\label{lem:sandwich}
Conditioned on $\mathcal{E}$, at any point of the run of Algorithm~\ref{alg:improved}, we have
$$\theta_K^- + \phi \ge \theta_{K - \frac{\eps^2K}{10}} \ge \theta_K \ge \theta_{K+1} \ge \theta_{K + 1+\frac{\eps^2 K}{10}}  \ge \theta_K^+ - \phi.$$
\end{lemma}

\begin{proof}
We consider a particular around $r$. The difference between $\theta_K^+$ and the fixed value $\theta_{K_R}$ is generated by calling of the subroutine {\sc EstKthArm} at Line~\ref{line:a-1}, which can be bounded by the the error introduced when selecting the $(K_R - \abs{A})$-th largest arm in $S$ plus $\max\{\iota_A, \iota_B\}$.  By Theorem~\ref{thm:estK}, Lemma~\ref{lem:misclassify} and $\mathcal{E}$, we have
$\theta_K^+ \le \theta_{K_R - \tau_R} + \eps_R$,
where $\eps_R \le \phi$,  and
$$\tau_R \le \tau \cdot (K_R - \abs{A}) + \max\{\iota_A, \iota_B\} \le \frac{\eps^2}{100 r^2} \cdot (1+\eps^2) K + \frac{\eps^2 K}{40} \le \frac{\eps^2 K}{20} \le K_R - K - \frac{\eps^2 K}{10}.$$
We thus have $\theta_K^+ \le \theta_{K+1+\frac{\eps^2}{10K}} + \phi$.

Similarly, by Theorem~\ref{thm:estK}, Lemma~\ref{lem:misclassify} and $\mathcal{E}$, we have
$\theta_K^- \ge \theta_{K_L + \tau_L} - \eps_L$,
where $\eps_L \le \phi$, and $\tau_L \le \max\{\iota_A, \iota_B\} \le \frac{\eps^2 K}{40} \le K - K_L - \frac{\eps^2 K}{10}$.  We thus have $\theta_K^- \ge \theta_{K-\frac{\eps^2 K}{10}} - \phi$.
\end{proof}

We have the following immediate corollary.
\begin{corollary}
\label{cor:sandwich}

If $\theta_K^- - \theta_K^+ \le 3 \phi$, then $\theta_K, \theta_{K+1} \in [\theta_K^+ - \phi, \theta_K^+ + 4\phi]$ and $\theta_K, \theta_{K+1} \in [\theta_K^- - 4\phi, \theta_K^- + \phi]$.
\end{corollary}

In the following, for convenience, we always use $S[i]$ to denote the true-mean of the $i$-th arm (sorted decreasingly) in the current set $S$ during the run of Algorithm~\ref{alg:improved}, and use $S[i..j]$ to denote the set of true-means of the $(i, i+1, \ldots, j)$-th arms in $S$. Let $\K \triangleq K - \abs{A}$.  We call $S[1 .. \K]$ the {\em head} of $S$, and $S[\K+1 .. \abs{S}]$ the {\em tail} of $S$. The following claim follows directly from Lemma~\ref{lem:misclassify} and Lemma~\ref{lem:sandwich}.
\begin{claim}
\label{cla:sandwich}
At any point during the run of Algorithm~\ref{alg:improved}, it holds that $\theta_{K+\frac{\eps^2 K}{40}} \le S[\K] \le \theta_{K-\frac{\eps^2 K}{40}}$, and consequently $\theta_K^- + \phi \ge S[\K]  \ge \theta_K^+ - \phi$.
\end{claim}


\begin{lemma}
\label{lem:elim-condition}

Conditioned on $\mathcal{E}$, the conditions of Lemma~\ref{lem:elim} always hold during the run of Algorithm~\ref{alg:improved}, that is, we have
\begin{equation}
\label{eq:y-0}
\textstyle S[\K] - S\left[\frac{\abs{S} + \K}{2}\right] \ge \phi \ \ \text{and} \ \ \K \le \frac{2}{3}\abs{S}.
\end{equation}
\end{lemma}

\begin{proof}
For the first item of (\ref{eq:y-0}), by Theorem~\ref{thm:estK} we have
\begin{equation}
\label{eq:y-4}
\textstyle S\left[\frac{\abs{S} + \K}{2}\right] - \phi \le  \theta^+,
\end{equation}
which together with $\theta_K^+ - \theta^+ > 3\phi$ (testing condition at Line~\ref{line:a-0}) and $S[\K] \ge \theta_K^+ - \phi$ (Claim~\ref{cla:sandwich}) give $S[\K] - S\left[\frac{\abs{S} + \K}{2}\right] \ge \phi$ (by two triangle inequalities).

For the second item of (\ref{eq:y-0}), note that if $\K \le \frac{\abs{S}}{2}$,  then $\K \le \frac{2}{3}\abs{S}$ holds directly.  We thus consider the case $\K > \frac{\abs{S}}{2}$. The observation is that at the beginning, before the first call to {\sc ReverseElim}, we must have called {\sc Elim} a number of times; each time we remove $\frac{\abs{S}}{10}$ arms, most of which are from the tail of $S$.  After the first time when $\K > \frac{\abs{S}}{2}$, we call both {\sc Elim} and {\sc ReverseElim}, with the intention of removing $\frac{\abs{S}}{10}$ arms from the tail and the  head respectively. It may happen that after calling both {\sc Elim} and {\sc ReverseElim} a few times, we again have $\K \le \frac{\abs{S}}{2}$, at which point we will again only call {\sc Elim} until the point that we are back to the case that $\K > \frac{\abs{S}}{2}$ and then we will call both {\sc Elim} and {\sc ReverseElim}.  Basically, the two patterns `call {\sc Elim} only' and `call both {\sc Elim} and {\sc ReverseElim}' interleave, and we only need to consider one run of this interleaved sequence.

By Lemma~\ref{lem:elim} and Lemma~\ref{lem:reverse-elim} we know that at most $\frac{\eps^2}{100 r^2} \K$ arms in the head of $S$ will be removed when calling {\sc Elim}, and at most $\frac{\eps^2}{100 r^2} \K$ arms in the tail of $S$ will be removed when calling {\sc ReverseElim}.  Therefore, the worst case for causing the imbalance between $\K$ and $(\abs{S} - \K)$ is that each call of {\sc Elim} removes $\frac{\abs{S}}{10}$ from the tail of $S$, and each call of {\sc ReverseElim} removes $\left(\frac{\abs{S}}{10} - \frac{\eps^2}{100 r^2} \K\right)$ arms from the head of $S$ and $\frac{\eps^2}{100 r^2} \K$ arms from the tail of $S$.  Note that the number of calls of {\sc Elim} and {\sc ReverseElim} is bounded by $O(\log \frac{1}{\eps})$ since when reaching Line~\ref{line:c-1} we always have $\abs{S} \ge \Omega(\phi) \ge \Omega(\eps)$.  We thus have
\begin{equation*}
\frac{\abs{S} - \K}{\abs{S}} \ge  \left(\frac{1}{2} - 0.1\right) \cdot \prod_{r=1}^{O(\log 1/\eps)} \left(1 - \frac{\eps^2}{100 r^2} \right) \ge \frac{1}{3},
\end{equation*}
which implies $\K \le \frac{2}{3}\abs{S}$.
\end{proof}

\begin{lemma}
\label{lem:reverse-elim-condition}

Conditioned on $\mathcal{E}$, the conditions of Lemma~\ref{lem:reverse-elim} always hold during the run of Algorithm~\ref{alg:improved}, that is, we have
\begin{equation}
\label{eq:y-5}
\textstyle S\left[\frac{\K}{2}\right] - S[\K] \ge \phi \ \ \text{and} \ \ \K \ge \frac{\abs{S}}{3}.
\end{equation}
\end{lemma}


\begin{proof}
In Algorithm~\ref{alg:improved}, when calling {\sc ReverseElim}, we always have
\begin{equation}
\label{eq:x-1}
\theta^- - \theta_K^-  > 3\phi  \ \ \text{and} \ \ \K  > \frac{\abs{S}}{2}.
\end{equation}
Thus $\K \ge \frac{\abs{S}}{3}$ follows directly.  By Theorem~\ref{thm:estK} we have
\begin{equation}
\label{eq:y-4}
\textstyle S\left[\frac{\K}{2}\right] + \phi \ge  \theta^-,
\end{equation}
which together with $\theta^- - \theta_K^- > 3\phi $ (first item of (\ref{eq:x-1})) and $S[\K] \le \theta_K^- + \phi$ (Claim~\ref{cla:sandwich}), gives $S\left[\frac{\K}{2}\right] - S[\K] \ge \phi$ (by two triangle inequalities).
\end{proof}

Lemma~\ref{lem:elim-condition}, Lemma~\ref{lem:reverse-elim-condition} and Lemma~\ref{lem:misclassify} give the following corollary.

\begin{corollary}
\label{cor:misclassify}
Conditioned on $\mathcal{E}$, during the run of Algorithm~\ref{alg:improved} we always have (1) the number of non-top-$K$ arms in $A$ is no more than $\iota_A = \frac{\eps^2 K}{40}$, and (2) the number of top-$K$ arms in $B$ is no more than $\iota_B = \frac{\eps^2 K}{40}$.
\end{corollary}

We now consider the boundary cases.
At Line \ref{line:b-1} when the condition is met, we have $\phi < \frac{\eps K}{10(K - \abs{A})}$, and thus with probability $(1 - \frac{\delta}{100})$ the total error introduced by subroutine {\sc OptMAI} at Line~\ref{line:b-4} is bounded by $(K - \abs{A}) \phi \le \frac{\eps K}{10}$ (Lemma~\ref{lem:optimal}).
At Line \ref{line:b-3}, with probability $(1 - \frac{\delta}{100})$ the error introduced by subroutine {\sc EpsSplit} is bounded by $2 \cdot \frac{K_R - K_L}{K - \abs{A}} \cdot (K - \abs{A}) \le 2 \cdot 2\eps^2 K = 4 \eps^2 K$. (Lemma~\ref{lem:gap}).

By $\mathcal{E}$, Corollary~\ref{cor:misclassify}, and the errors introduced by boundary cases, we have that with probability $1 - \frac{\delta}{5} - 2 \cdot \frac{\delta}{100} \ge 1 - \delta$, the total error introduced in our top-$K$ estimation is at most  $\iota_A + \iota_B + \frac{\eps K}{10} + 4 \eps^2 K \le \eps K$.

\subsubsection{Complexity of Algorithm~\ref{alg:improved}}
\label{sec:complexity-improved}

In the whole analysis we assume that $\mathcal{E}$ holds. Recall that by definition $\Delta_i = \max(\theta_i - \theta_{K+1}, \theta_K - \theta_i)$, and $t\in[K]$ is the largest integer such that $\Delta_{K - t} \cdot t \leq K\eps$ and $\Delta_{K+1+t} \cdot t \leq K\eps$.  Recall that $\Psi_t = \min\{\Delta_{K - t}, \Delta_{K+1+t}\}$, and $\Psi_t^{\eps} = \max(\eps, \Psi_t)$.

By Theorem~\ref{thm:estK}, Lemma~\ref{lem:elim}, Lemma~\ref{lem:reverse-elim}, Lemma~\ref{lem:optimal} and Lemma~\ref{lem:gap} we have: every call to {\sc EstKthArm} costs $O\left(\frac{\abs{S}}{\phi^2} \log \frac{r r_\phi}{\eps \delta}\right)$ pulls; every call to {\sc Elim} and {\sc ReverseElim} costs $O\left(\frac{\abs{S}}{\phi^2} \log \frac{r}{\eps \delta}\right)$; the call to {\sc OptMAI} costs $O\left(\frac{\abs{S}}{\phi^2} \log \frac{1}{ \delta}\right)$; and the call to {\sc EpsSplit} costs $O\left(\frac{\abs{S}}{\phi^2} \log \frac{\K}{\eps \delta K}\right) = O\left(\frac{\abs{S}}{\phi^2} \log \frac{1}{\eps \delta}\right)$.  So our task is to lower bound the value of $\phi$ when these subroutines are called, and the maximum values of $r_\phi$ and $r$.
\begin{lemma}
\label{lem:bound-phi}

Conditioned on $\mathcal{E}$, at any point of the run of Algorithm~\ref{alg:improved}, we have $\phi = \Omega(\Psi_t^\eps)$, and $r_\phi = O(\log \frac{1}{\eps}), r = O(\log\frac{n}{\abs{S}})$.
\end{lemma}

\begin{proof}
First, by the testing condition $10(K - \abs{A})\phi < K\eps$ at Line~\ref{line:a-0}, together with the boundary cases at Line~\ref{line:b-1}-\ref{line:b-4} and the fact that $\phi \gets \phi/2$ at every update, it holds that
\begin{equation}
\label{eq:w-0}
10(K - \abs{A}) \cdot 2\phi \ge K\eps,
\end{equation}
which implies $\phi = \Omega(\eps)$ when we call all the subroutines.

We now show $\phi = \Omega(\Psi_t)$.
From the proof of Lemma~\ref{lem:elim-condition} we know that during the run of the Algorithm we always have $\K \le \frac{2\abs{S}}{3}$. We focus on an arbitrary but fixed point during the run of the algorithm.  By Theorem~\ref{thm:estK} we have
\begin{equation}
\label{eq:w-8}
\textstyle S\left[\frac{\abs{S}+\K}{2}\right] - \phi \le \theta^+ \le S\left[\left(1  - \frac{\eps^2}{100r^2}\right)\frac{\abs{S}+\K}{2}\right] + \phi,
\end{equation}
which, combined with the fact that $\K \le \frac{2\abs{S}}{3}$, gives
\begin{equation}
\label{eq:w-9}
\textstyle S\left[\frac{\abs{S}+\K}{2}\right] - \phi \le \theta^+ \le S\left[\frac{\abs{S}+\K - \eps\abs{S}}{2}\right] + \phi.
\end{equation}
By (\ref{eq:w-0}) we have
\begin{equation}
\label{eq:w-6}
\textstyle \frac{\eps^2 K}{40} \le \frac{\eps}{40} \cdot 10 \K \cdot 2\phi \le \frac{\eps \K}{2}.
\end{equation}
Applying Claim~\ref{cla:sandwich} on both sides of (\ref{eq:w-8}), together with (\ref{eq:w-6}) and $\K \le \frac{2\abs{S}}{3}$ we have
\begin{eqnarray}
&&\theta_{K + \frac{\abs{S}-\K}{2} + \frac{\eps^2 K}{40}} - \phi \le \theta^+ \le \theta_{K + \frac{\abs{S}-\K - \eps \abs{S}}{2} - \frac{\eps^2 K}{40}} + \phi \nonumber \\
&\Rightarrow& \theta_{K + \frac{\abs{S}-\K}{2} + \frac{\eps \K}{2}} - \phi \le \theta^+ \le \theta_{K + \frac{\abs{S}-\K - \eps \abs{S}}{2} -  \frac{\eps \K}{2}} + \phi \nonumber \\
&\Rightarrow& \textstyle \theta^+ \in \theta_{K + \eta} \pm \phi \ \ \text{for an} \ \ \eta \ge \frac{\abs{S} - \K}{2} - \eps \abs{S} \ge 0.16\abs{S}. \label{eq:w-1}
\end{eqnarray}

By symmetry, using a similar argument we can show that for a sufficiently small constant $c_{\eta'}$, we have
\begin{equation}
\label{eq:ww-1}
\theta^- \in \theta_{K - \eta'} \pm \phi \ \ \text{for an}\ \ \eta' \ge c_{\eta'} \abs{S},
\end{equation}

We first consider the case where $K - |A| \leq \frac{|S|}{2}$.  We analyze the following two sub-cases.
\begin{enumerate}
\item[1a)] The case when $\eta > t$.  We have:
\begin{align*}
\label{eq:w-2}
\phi \ge& \frac{\theta_K^+ - \theta^+}{6} \quad  (\text{by the testing condition at Line~\ref{line:a-0} and $\phi \gets \phi/2$ at each update)} \\
\ge& \frac{(\theta_K - 4 \phi) - (\theta_{K + \eta} + \phi)}{6} \quad (\text{by Corollary~\ref{cor:sandwich} and (\ref{eq:w-1})}) \\
\ge& \frac{\theta_K - 4 \phi - (\theta_{K + t+1} + \phi)}{6} \\
\ge& \frac{\Psi_t - 5\phi}{6} \quad (\text{by the definition of $\Psi_t$}).
\end{align*}

\item[1b)] The case when $\eta \le t$.  We prove by contradiction.  Suppose that $\phi \le \Psi_t/c_t$ for a sufficiently large constant $c_t$, then
\begin{equation*}
\label{eq:w-3}
\Psi_t \cdot t \ge c_t\phi \cdot \eta  \overset{(\ref{eq:w-1})}{\ge}  c_t\phi \cdot 0.16 \abs{S} \overset{\K \le \frac{2\abs{S}}{3}}{\ge} c_t\phi \cdot  0.16 \cdot 1.5 \K \overset{(\ref{eq:w-0})}{>} K\eps,
\end{equation*}
A contradition to the definition of $t$.
\end{enumerate}

We then consider the case where  $K - |A| > \frac{|S|}{2}$. Now by the testing condition at Line~\ref{line:a-0} and $\phi \gets \phi/2$ at each update, we know that at least one of the following inequality holds:  $\phi \ge \frac{\theta_K^+ - \theta^+}{6}$; or  $\phi \ge \frac{\theta^- - \theta_K^-}{6}$. If the first inequality holds, the case-analysis above suffices. Otherwise, we know that the second inequality holds, we analyze the following two sub-cases in a similar fashion.
\begin{enumerate}
\item[2a)] The case when $\eta' > t$.
\begin{align*}
\phi \ge& \frac{\theta^- - \theta_K^-}{6} \\
\ge& \frac{ (\theta_{K - \eta'} - \phi) - (\theta_{K+1} + 4 \phi)}{6} \quad (\text{by Corollary~\ref{cor:sandwich} and (\ref{eq:ww-1})}) \\
\ge& \frac{ (\theta_{K - t} - \phi) - (\theta_{K+1} + 4 \phi)}{6} \\
\ge& \frac{\Psi_t - 5\phi}{6} \quad (\text{by the definition of $\Psi_t$}).
\end{align*}

\item[2b)] The case when $\eta' \leq t$. This case is symmetric to Case 1b), and we omit here.
\end{enumerate}

Since $\phi = 2^{-r_\phi}$, we immediately have $r_\phi =  O(\log\frac{1}{\Psi_t^\eps}) = O(\log \frac{1}{\eps})$.  By the testing condition at Line~\ref{line:b-2}, and the fact that every time we call {\sc Elim} and {\sc ReverseElim} we remove a constant fraction of arms from $S$, we thus have $r = O(\log \frac{n}{\abs{S}})$.
\end{proof}

We now look at a particular call to {\sc Elim} which removes $\frac{1}{10}$-fraction of arms in $S$, and the tail of $S$.  From the testing condition at Line~\ref{line:a-0} we know that
\begin{equation}
\label{eq:u-1}
\theta_K^+ - \theta^+ \le 3 \cdot 2\phi.
\end{equation}
From Corollary~\ref{cor:sandwich} we have that
\begin{equation}
\label{eq:u-2}
\theta_K \le \theta_K^+ + 4\phi.
\end{equation}
From (\ref{eq:u-1}), (\ref{eq:u-2}) and the second inequality of (\ref{eq:w-9}), by applying two triangle inequalities we have
\begin{equation}
\label{eq:u-3}
\textstyle \theta_K - S\left[\frac{\abs{S} + \K - \eps\abs{S}}{2}\right] \le 11\phi,
\end{equation}
which implies that for all
\begin{equation}
\label{eq:u-31}
\textstyle j \in Q = \left[1.01 \K, \frac{\abs{S} + \K - \eps\abs{S}}{2}\right] (\abs{Q} \ge \frac{\abs{S}}{10}\ \ \text{since} \ \ \K \le \frac{2\abs{S}}{3}),
\end{equation}
letting $m(j) \in [n]$ such that $\theta_{m(j)} = S[j]$, we have
\begin{equation}
\Delta_{m(j)} = \theta_K - S\left[ j \right] \le 11\phi.
\end{equation}
We thus can charge all the previous cost spent on the $\frac{1}{10}$-fraction of arms in $S$ that are removed by {\sc Elim}, which is bounded by
$O\left(\frac{\abs{S}}{\phi^2} \left(\log \lceil \log\frac{n}{\abs{S}} \rceil + \log \frac{1}{\eps} + \log \frac{1}{\delta}\right)\right)$,
to
\begin{equation}
\label{eq:u-4}
O\left(\sum_{j \in Q} \frac{1}{\Delta_{m(j)}^2} \left(\log \left\lceil \log \frac{n}{m(j) - K} \right\rceil + \log \frac{1}{\eps} + \log  \frac{1}{\delta} \right)\right),
\end{equation}
where we have used the fact that
\begin{align*}
\abs{S} ~=~& \Omega(j - \tilde{K})  \quad \text{(by (\ref{eq:u-31}))} \\ ~\ge~& \Omega(m(j) - K). \quad \text{(by (\ref{eq:u-31}), Corollary~\ref{cor:misclassify} and (\ref{eq:w-6}))}
\end{align*}
Note that it is possible that in multiple calls to {\sc Elim} with parameters $(S_1, \cdot, \cdot, \phi_1, \cdot), \ldots, (S_{\kappa}, \cdot, \cdot, \phi_{\kappa}, \cdot)$ where $\phi_1 \ge \ldots \ge \phi_{\kappa}$, we charge the same item $j \in Q_1 \cap \ldots \cap Q_{\kappa}$ multiple times. However, since $\phi_{i+1} \le \phi_i/2$ for all $i \in [\kappa-1]$, the total charge on $j$ is at most twice of that of the last charge (i.e., the one with parameter $\phi_\kappa$).

By symmetry, we can use the same arguments for {\sc ReverseElim} and the head of $S$, and get a same bound as (\ref{eq:u-4}) except that we need to replace $m(j) - K$ with $K - m(j)$.
We thus conclude that the total number of pulls can be bounded by
\begin{equation}
\label{eq:t-1}
O\left(\sum_{j \in [n]} \frac{1}{\Delta_j^2} \left(\log \left\lceil \log \frac{n}{\abs{j - K + \frac{1}{2}}}\right\rceil + \log \frac{1}{\eps} + \log  \frac{1}{\delta} \right) \right) .
\end{equation}
We know from Lemma~\ref{lem:bound-phi} that we always have $\phi = \Omega(\Psi_t^\eps)$, we can thus ``truncate'' Expression~(\ref{eq:t-1}) and bound the total cost by
\begin{equation}
\label{eq:t-2}
O\left(\sum_{j \in [n]} \frac{1}{\max\{\Delta_j^2, (\Psi_t^\eps)^2\}} \left(\log \left\lceil \log \frac{n}{\abs{j - K + \frac{1}{2}}} \right\rceil + \log \frac{1}{\eps} + \log  \frac{1}{\delta} \right) \right) .
\end{equation}
Now we introduce the following lemma (the proof of which is deferred to the Appendix).
\begin{lemma}
\label{lem:sum}
If  $M > a_1 \ge \ldots \ge a_n \ge 1$, then $\sum_{i \in [n]} a_i \log (n/i) \le O(\lceil \log M \rceil) \sum_{i \in [n]} a_i$.
\end{lemma}

With Lemma~\ref{lem:sum} we can further simplify (\ref{eq:t-2}) to
\begin{align*}
&O\left(\sum_{j \in [n]} \frac{1}{\max\{\Delta_j^2, (\Psi_t^\eps)^2\}} \left(\log\max\left\{ \frac{1}{\max\{\Delta_j^2, (\Psi_t^\eps)^2\}}\ |\ j \in [n]\right\} + \log \frac{1}{\eps} + \log  \frac{1}{\delta} \right) \right) \\
= \  &O\left(\sum_{j \in [n]} \frac{1}{\max\{\Delta_j^2, (\Psi_t^\eps)^2\}} \left(\log \frac{1}{\eps} + \log  \frac{1}{\delta} \right) \right) .
\end{align*}

\section{A Lower Bound}

In this section we prove Theorem~\ref{thm:intro-lb-eps}. In Section~\ref{sec:two-point}, we introduce a lower-bound to a coin-tossing problem. In Section~\ref{sec:reduction}, we reduce the proof of Theorem~\ref{thm:intro-lb-eps} to the coin-tossing problem.

\subsection{The Coin-Tossing Problem}
\label{sec:two-point}

We say a coin is {\em $p$-biased} if the probability that a toss turns {\em head} is $p$, and we call $p$ is the {\em value} of the coin.  Set $\eta = 10^{-4}$.
\begin{definition}[Coin-Tossing]
\label{def:two-point}
In this problem, given a coin that may be $(0.5+\eta)$-biased or $(0.5-\eta)$-biased, we want to know its exact value by tosses, and we are allowed to give up and output `unknown' with probability at most $0.9$.
\end{definition}
  We have the following theorem.

\begin{theorem}
\label{thm:two-point}
Any algorithm that solves the \two\ problem correctly with probability $(1 - \eps)$ needs $\Omega(\log 1/\eps)$ tosses.
\end{theorem}

\begin{proof}
Since the input is distributional we only need to focus on deterministic algorithms.
Let $m$ be the total number of tosses of the coin, and let $\vB = (B_1, \ldots, B_m) \in \{0,1\}^m$ be the sequence of outcomes.   Let $\D_\beta$ be the distribution of $\vB$ where each $B_i$ is the outcome of tossing a $\beta$-biased coin.
For $\mathbf{v} \in \{0,1\}^m$, let $\abs{\mathbf{v}}$ be the number of $1$-coordinates in $\mathbf{v}$.

Our first observation is  that for any $\vb_1, \vb_2 \in \{0,1\}^m$, if  $\abs{\vb_1} = \abs{\vb_2}$, then $\Pr[\vB = \vb_1] = \Pr[\vB = \vb_2]$.  Therefore, the final output should only depend on the value $\abs{\vB}$  but {\em not} the ordering of the $0/1$ sequence.  In other words, we can view the output of the algorithm as a function
\begin{equation*}
f : \{0, 1, \ldots, m\} \to \{0.5-\eta, 0.5+\eta, \perp\},
\end{equation*}
where $\{0, 1, \ldots, m\}$ stand for possible values of $\abs{\vB}$, and `$\perp$' represents `unknown'.  Recall that the algorithm can give up and output `unknown' with probability at most $0.9$.  By observing that $\vB \sim \D_{0.5+\eta}$ and $\vB \sim \D_{0.5-\eta}$ are symmetric, the best strategy must set $f(x) = \perp$ for $x \in [0.5m - t, 0.5m + t]$, where $t \in \mathbb{N}$ is the maximum value such that
\begin{equation}
\label{eq:b-1}
\Pr_{\vB \sim \D_{0.5-\eta}} [0.5m - t \le \abs{\vB} \le 0.5m + t] \le 0.9.
\end{equation}
Intuitively, $[0.5m - t, 0.5m + t]$ is the range where conditioned on  $\abs{\vB} \in [0.5m - t, 0.5m + t]$ the value of the coin is the most uncertain (so that the algorithm simply outputs `$\perp$'). We set $f(x) = 0.5-\eta$ if $x \in [0, 0.5m - t)$, and $f(x) = 0.5+\eta$ if $x \in (0.5m + t, m]$.  The error probability of this strategy is
\begin{equation}
\label{eq:b-2}
\Pr_{\vB \sim \D_{0.5-\eta}} [\abs{\vB} > 0.5m + t].
\end{equation}
We now try to upper bound $t$.  First, it is easy to see that $0.5m - t \le (0.5-\eta) m$, or $t \ge \eta m$, since otherwise LHS of (\ref{eq:b-1}) is at most $1/2$, violating the choice of $t$. By a Hoeffding's inequality we have
\begin{equation*}
\label{eq:b-3}
\Pr_{\vB \sim \D_{0.5-\eta}} [\abs{\vB} \le 0.5m - t] = \Pr_{\vB \sim \D_{0.5-\eta}} [\abs{\vB} \le \bE[\abs{\vB}] - (t - \eta m)] \le e^{-m(t/m - \eta)^2/2}.
\end{equation*}
We thus have $e^{-m(t/m - \eta)^2/2} \ge (1-0.9)/2$, and consequently
\begin{equation}
\label{eq:b-4}
t \le \eta m + c_t \sqrt{m}
\end{equation}
for some large enough constant $c_t$.

We now lower bound the expression (\ref{eq:b-2}).  We will need the following anti-concentration result which is an easy consequence of Feller~\cite{feller:43} (cf. \cite{Matousek:08}).

\begin{fact}(\cite{Matousek:08})
\label{lem:feller}
Let $Y$ be a sum of independent random variables, each attaining values in $[0,1]$, and let $\sigma = \sqrt{\var[Y]} \ge 200$. Then for all $t \in [0, \sigma^2/100]$, we have
$$\Pr[Y \ge \bE[Y] + t] \ge c\cdot e^{-t^2/(3\sigma^2)}$$
for a universal constant $c > 0$.
\end{fact}
In our case, since $\abs{\vB}$ can be seen as a sum of Bernoulli variables with $p = 0.5-\eta$, $\var{\abs{\vB}} = m \cdot (0.5-\eta) \cdot (0.5+\eta) \ge 0.24m$.  By (\ref{eq:b-4}) we have $\eta m + t \le 2\eta m + c_t \sqrt{m} \le \var{\abs{\vB}}/100$ by our choice of $\eta$.  Thus by applying Fact~\ref{lem:feller}
we have
\begin{eqnarray*}
\Pr_{\vB \sim \D_{0.5-\eta}} [\abs{\vB} > 0.5m + t]
&=& \Pr_{\vB \sim \D_{0.5-\eta}} [\abs{\vB} > \bE[\abs{\vB}] + (\eta m + t)]  \\
&\ge& c \cdot e^{-(\eta m+t)^2/(3 \cdot 0.24m)} \\
&\ge& e^{-\Omega(m)}.
\end{eqnarray*}
To make the best strategy succeeds with probability at least $1-\eps$, we have to make $e^{-\Omega(m)} \le \eps$, which means we have to set $m \ge \Omega(\log 1/\eps)$.
\end{proof}

\subsection{The Reduction}
\label{sec:reduction}

We show a reduction from the \two\ problem to the \arm\ problem.  For technical convenience we set $K = n/2$, and assume that $\eps K \ge c_K$ for a large enough constant $c_K$.

\begin{lemma}
\label{lem:reduction}
If there is an algorithm for \arm\ that succeeds with probability $0.9$ using $C \le f(n, K) / \text{poly}(\eps)$ tosses, then there is an algorithm for \two\ that succeeds with $(1 - \eps)$ using $O(C/n)$ tosses. Moreover, the instances fed into the \arm\ algorithm have the property that $H^{(t, \epsilon)} = \Theta(n \eta^{-2}) = \Theta(n)$ for $\epsilon \geq c_K / K$.
\end{lemma}

We prove Lemma~\ref{lem:reduction} in two steps.   We first perform an input reduction, and then show that we can construct an efficient algorithm for \two\ using an algorithm for \arm.

\paragraph{Input reduction}
Given an input $X$ for \two, we construct an input $Y = (X_1, \ldots, X_n)$ for \arm\ as follows: we randomly pick a set $S \subseteq [n]$ with $\abs{S} = K$, and set $X_i\ (i \in S)$ to be $(0.5+\eta)$-biased coins (denoted by $X_i = 0.5+\eta$ for convenience), and $X_i\ (i \neq [n] \backslash S)$ to be $(0.5-\eta)$-biased coins.  We then pick a random index $I \in [n]$, and reset $X_I = X$. Since in our input $Y$, the number of $(0.5+\eta)$-biased coins is either $K-1$, $K$, or $K+1$, while the rest are $(0.5-\eta)$-biased coins, it can be checked that $H^{(t, \epsilon)} = \Theta(n \eta^{-2})$ for $\epsilon \in [c_K/K, \eta]$.

\begin{claim}
\label{cla:embed}

If $S'$ is a set of \armG\ ($\gamma \ge 2\eta/K$) on $Y$, then with probability at least $1 -  2\gamma/\eta$ we can correctly determine the value of $X$ by checking whether $I \in S'$.
\end{claim}

\begin{proof}
If $X = X_I = 0.5+\eta$, then to compute \arm\ correctly we need to output a set $S'$ such that
\begin{equation*}
 \frac{1}{K} \cdot \sum_{i \in S'} X_i \ge  \frac{1}{K} \cdot \sum_{i \in S} X_i - \gamma \ge (0.5+\eta) - \gamma.
\end{equation*}
By simple calculation we must have $\abs{\{i \in S'\ |\ X_i = 0.5+\eta\}} \ge \left(1 - \frac{\gamma}{2\eta}\right) K$.  Since all $(0.5+\eta)$-biased coins are {\em symmetric}, the probability that $I \in S'$ is at least
\begin{equation}
\label{eq:a-1}
\frac{\left(1 - \frac{\gamma}{2\eta}\right)}{K+1} \ge \left(1 - \frac{\gamma}{\eta}\right).
\end{equation}

Otherwise if $X = X_I = 0.5-\eta$, then to compute \arm\ correctly we need to output a set $S'$ such that
\begin{equation*}
 \frac{1}{K} \cdot \sum_{i \in S'} X_i \ge  \frac{1}{K} \cdot \sum_{i \in S} X_i - \gamma - \frac{2\eta}{K}  \ge 0.5+\eta  - \gamma - \frac{2\eta}{K},
\end{equation*}
By simple calculation we must have $\abs{\{i \in S'\ |\ X_i = 0.5+\eta\}} \ge \left(1 - \frac{\gamma}{\eta}\right) K$, or,  $\abs{\{i \in S'\ |\ X_i = 0.5-\eta\}} \le \frac{\gamma}{\eta} K$. Again since all $(0.5-\eta)$-biased coins are symmetric, the probability that $I \in S'$ is at most
\begin{equation}
\label{eq:a-2}
\frac{\frac{\gamma}{\eta} K}{n - (K-1)} \le \frac{2\gamma}{\eta}.
\end{equation}
By (\ref{eq:a-1}) and (\ref{eq:a-2}), we conclude that by observing whether $I \in S'$ or not we can determine whether $X = X_I = 0.5+\eta$ or $0.5-\eta$ correctly with probability at least $1 - 2\gamma/\eta$.
\end{proof}

\paragraph{An algorithm for \two}
Let $\eps' = \eta/4 \cdot \eps$.  We now construct an algorithm $\A'$ for \two\ using an algorithm $\A$ for \armT.  Given an input $X$ for \two, we first perform the input reduction as described above, getting an input $Y$. We then run $\A$ on $Y$.  We give up and output `unknown' during the run of $\A$ if the number of tosses on $X_I$ is more than $20C/n$.  If $\A$ finishes then let $S'$ be the outputted set of top-$K$ coins.  We then perform a verification step to test whether $S'$ is indeed a set of \armT, and output `unknown' if the verification fails.  The verification is done as follows: we first compute $\rho = \frac{1}{\abs{S' \backslash I}} \cdot \sum_{i \in S' \backslash I} X_i$, and then verify whether $\rho \ge (0.5+\eta) - (\eps' + 2\eta/K).$
Finally, if we have not outputted `unknown', we output $X = 0.5+\eta$ if $I \in S'$, and $X = 0.5-\eta$ if $I \not\in S'$.
\medskip

We now try to bound the probability that $\A'$ outputs `unknown'.

\begin{claim}
\label{cla:give-up}

The probability that we give up during the run of $\A$ is at most $0.1$.
\end{claim}

\begin{proof}
We prove for the case $X = 0.5+\eta$; same arguments hold for the case $X = 0.5-\eta$ since we  have set $K = n/2$.  Note that if $I \in S$, then we have $K$ coins (including $X_I$) in $Y$ that has value $(0.5+\eta)$; otherwise if $I \not\in S$ then we have $K+1$ such coins. By symmetry, the expected tosses on $X_I$ is at most $C/K = 2C/n$.  By a Markov inequality the probability that $X_I$ has been tossed by at least $20C/n$ is at most $0.1$.
\end{proof}

\begin{claim}
\label{cla:verification}

Suppose we do not give up when running $\A$, the verification step fails with probability at most $0.1$.
\end{claim}

\begin{proof}
Note that $\A'$ knows the value of all other coins except $X_I$, simply because $\{X_i\ |\ i \in [n] \backslash I\}$ are all constructed by $\A'$. The $2\eta/K$ factor in the test $\rho \ge (0.5+\eta) - (\eps' + 2\eta/K)$ comes from the fact that we do not know the exact value of $X_I$, which will affect the estimation of the $\frac{1}{K} \sum_{i \in S'} X_i$ by at most an additive factor $2\eta/K$.  Therefore the failure probability of the verification is at most the failure probability of $\A$, which is upper bounded by $0.1$.
\end{proof}

%

Now we are ready to prove the lemma.
\begin{proof}(of Lemma \ref{lem:reduction})
First, note that if there is an algorithm for \arm\ that succeeds with probability $0.9$ using $C \le f(n, K) \cdot \text{poly}(\eps)$ tosses, then there is an algorithm for \armT\ ($\eps' = 4/\eta \cdot \eps$ for a constant $\eta$) that succeeds with probability $0.9$ using $O(C)$.

We now show that Algorithm $\A'$ constructed above for \two\ has the following properties, which conclude the lemma.
\begin{enumerate}
\item It tosses $X$ at most $O(C/n)$ times.

\item It outputs `unknown' with probability at most $0.9$.

\item When it does not output `unknown', it successfully computes $X$ with probability at least $1 - \eps$.
\end{enumerate}
The first item holds according to the construction of $\A'$.  For the second item, the probability that $\A'$ outputs `unknown' is upper bounded by the sum of the probability that we give up when running $\A$ and the failure probability of $\A$, which is at most $0.1 + 0.1 < 0.9$ by Claim~\ref{cla:give-up} and Claim~\ref{cla:verification}.   For the third item, note that any $S'$ that passes the verification step in $\A'$ is a set of $(\eps' + 2\eta/K)$-top-$K$ arms.  The item holds by applying Claim~\ref{cla:embed} (setting $\gamma = \eps' + 2\eta/K$). Note that $2\gamma/\eta = (2/\eta) \cdot (\eps' + 2\eta/K) \le \eps$ since $\eps' = \eta/4 \cdot \eps$ and $\eps K \ge c_K$ for a sufficiently large constant $c_K$.
\end{proof}

By Theorem~\ref{thm:two-point} and Lemma~\ref{lem:reduction} we have the following theorem.
\begin{theorem}
\label{thm:lb-eps}
Any algorithm that computes that \arm\ correctly with probability $0.9$ needs $\Omega(n \log \eps^{-1})$ tosses.
\end{theorem}
Together with the bound for $H^{(t, \epsilon)}$ in Lemma~\ref{lem:reduction}, we prove Theorem~\ref{thm:intro-lb-eps}.

\newcommand{\optmai}{{OptMAI}}
\newcommand{\clucb}{{CLUCB-PAC}}
\newcommand{\adptopk}{{AdaptiveTopK}}
\newcommand{\uniform}{{\sc Uniform}}
\newcommand{\synp}{{\sc Synthetic}-$p$}
\newcommand{\synfour}{{\sc Synthetic-4}}
\newcommand{\synsix}{{\sc Synthetic-6}}
\newcommand{\syntwo}{{\sc Synthetic-2}}
\newcommand{\synthree}{{\sc Synthetic-3}}
\newcommand{\synhalf}{{\sc Synthetic-.5}}
\newcommand{\twogroup}{{\sc TwoGroup}}
\newcommand{\rte}{{\sc Rte}}

\section{Experiments}
\label{sec:exp}
In this section we present the experimental results. While our theorems are presented in the PAC form, it is in general difficult to verify them  directly because the parameter $\eps$ is merely an upper bound and the actual aggregate regret may deviate from it.  In our experiment, we convert our  Algorithm~\ref{alg:basic} to the fixed-budget version (that is, fix the budget of the number of pulls and calculate the aggregate regret). We compare our  Algorithm~\ref{alg:basic} (\adptopk~)  with two state-of-the-art methods -- \optmai~in \cite{Zhou:14} and \clucb~ in \cite{Chen-Lin-King-Lyu-Chen-13}.  The comparison between \optmai~/\clucb~ and previous methods (e.g., the methods in \cite{Bubeck:13} and \cite{Kalyanakrishnan:12}) have already been demonstrated in \cite{Zhou:14} and  \cite{Chen-Lin-King-Lyu-Chen-13}, and thus omitted here for the clarity of the presentation.  
%
To convert our algorithm to the fixed-budget version, we remove the outer while loop of Algorithm \ref{alg:basic}. As a replacement, we keep track of the total number of pulls, and stop pulling the arms once the budget is exhausted. 

We test our algorithm on both synthetic and real datasets as described as follows. For simulated datasets, we set the total number of arms $n=1,000$ and vary the parameter $K$. We set the tolerance parameter $\epsilon=0.01$. In \adptopk~and \clucb, another parameter $\delta$  (i.e., the failure probability) is required and we set $\delta=0.01$.
\begin{itemize}
\item \twogroup: the mean reward for the top $K$ arms is set to $0.7$ and that for the rest of the arms is set to  $0.3$.
\item \uniform: we set $\theta_i = 1 - \frac{i}{n}$ for $1 \leq i \leq n$.
\item \synp: we set $\theta_i = (1 - \frac{K}{n}) + \frac{K}{n}\cdot(1 - \frac{i}{K})^p$ for each $i \leq K$ and $\theta_{i} = (1 - \frac{K}{n}) - \frac{n-K}{n}\cdot(\frac{i-K}{n-K})^p$ for each $i > K$. Note that {\sc  Synthetic}-$1$ is identical to \uniform. When $p$ is larger than $1$, arms are made closer to the boundary that separates the top-$K$ from the rest (i.e. $1 - \frac{K}{n}$). When $p$ is smaller than $1$, arms are made farther to the boundary. We normalize all the arms such that the mean values of the arms still span the whole interval $[0, 1]$. We consider $p=.5, 1, 6$.

\item \rte: We generate $\theta$ from a real recognizing textual entailment (RTE) dataset \cite{Snow:2008}. There are $n = 164$ workers and we set each $\theta_i$ be the true labeling accuracy of the $i$-th worker. Note that the true label for each instance is provided  in this dataset.
\end{itemize}

For each dataset, we first fix the budget (total number of pulls allowed) and run each algorithm $200$ times. For each algorithm, we calculate the empirical probability (over 200 runs) that the aggregate regret of the selected arms is  above the tolerance threshold $\eps=0.01$, which is called \emph{failure probability}. A smaller failure probability means better performance. For each dataset and different $K$, we plot the curve of failure probability by varying the number of pulls. The results are shown in Figure \ref{fig:twogroup}-\ref{fig:rte}.

\begin{figure*}[!ht]
     \centering
     \subfloat[][$K=100$]{\includegraphics[height=0.19\textwidth,width=0.25\textwidth]{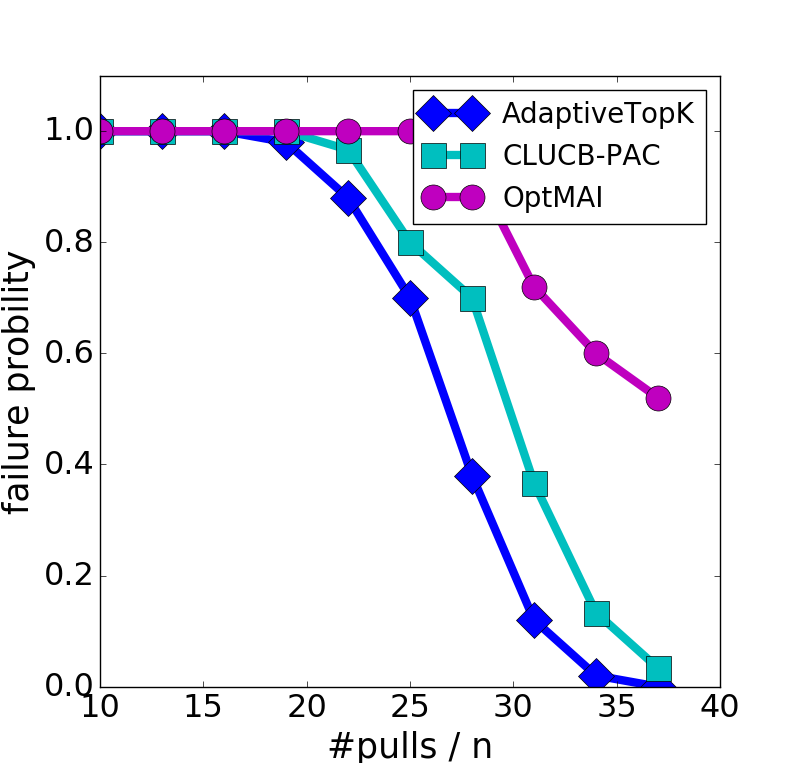}}
     \subfloat[][$K=250$]{\includegraphics[height=0.19\textwidth,width=0.25\textwidth]{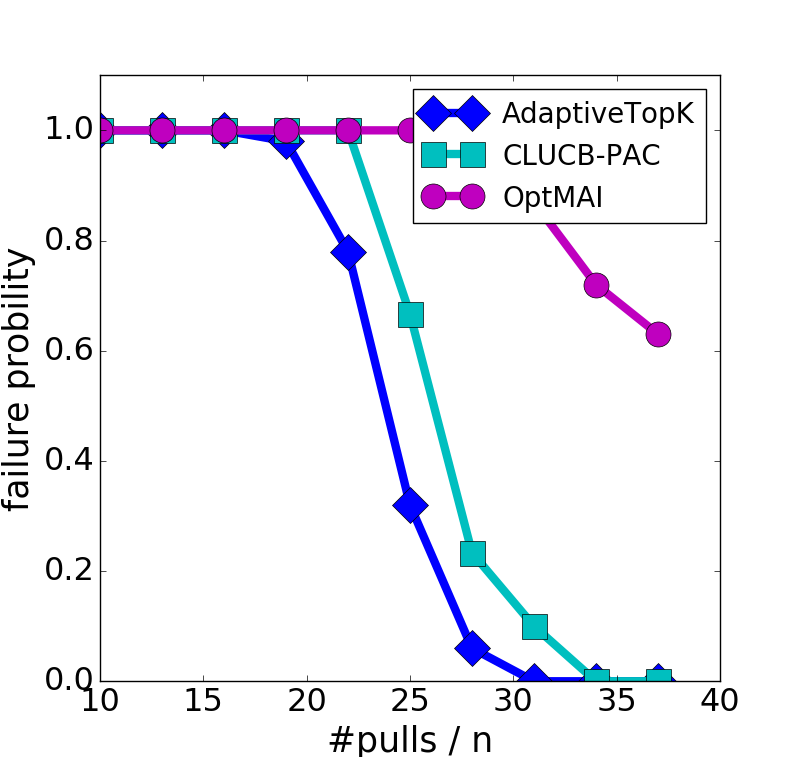}}
     \subfloat[][$K=500$]{\includegraphics[height=0.19\textwidth,width=0.25\textwidth]{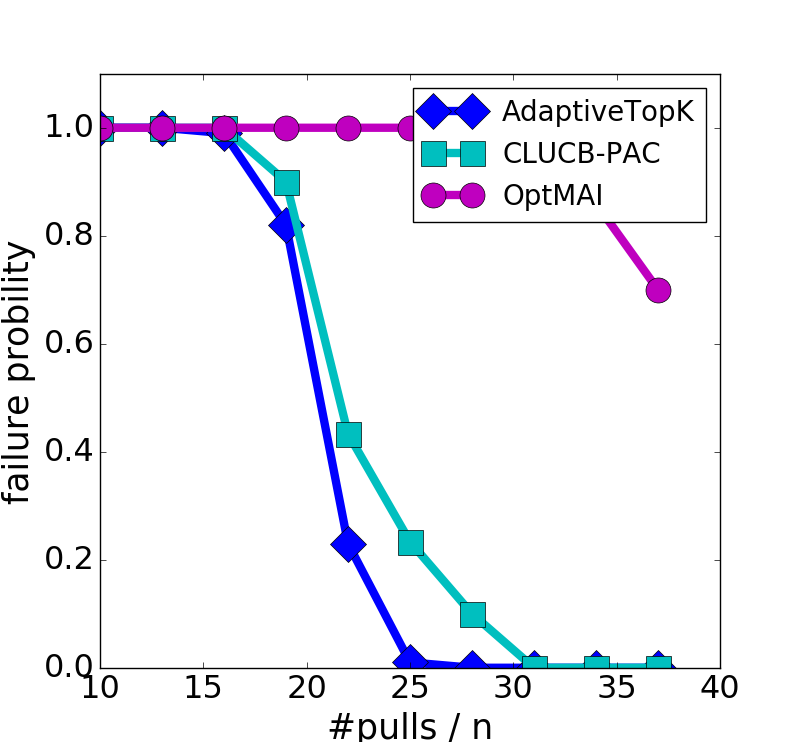}}
     \caption{\twogroup~dataset}
     \label{fig:twogroup}
\end{figure*}

\begin{figure*}[!ht]
     \centering
     \subfloat[][$K=100$]{\includegraphics[height=0.19\textwidth,width=0.25\textwidth]{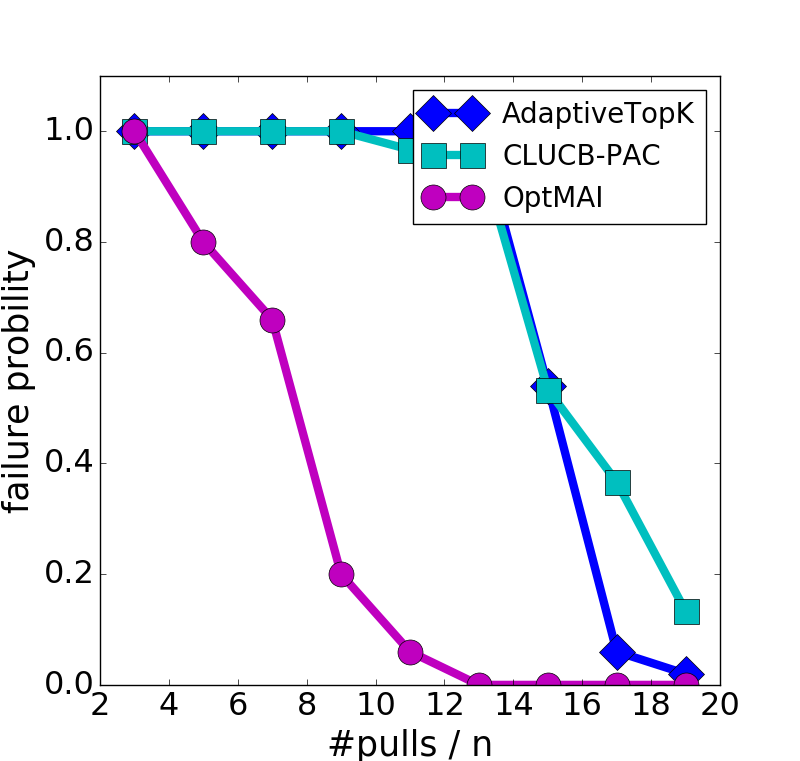}}
     \subfloat[][$K=250$]{\includegraphics[height=0.19\textwidth,width=0.25\textwidth]{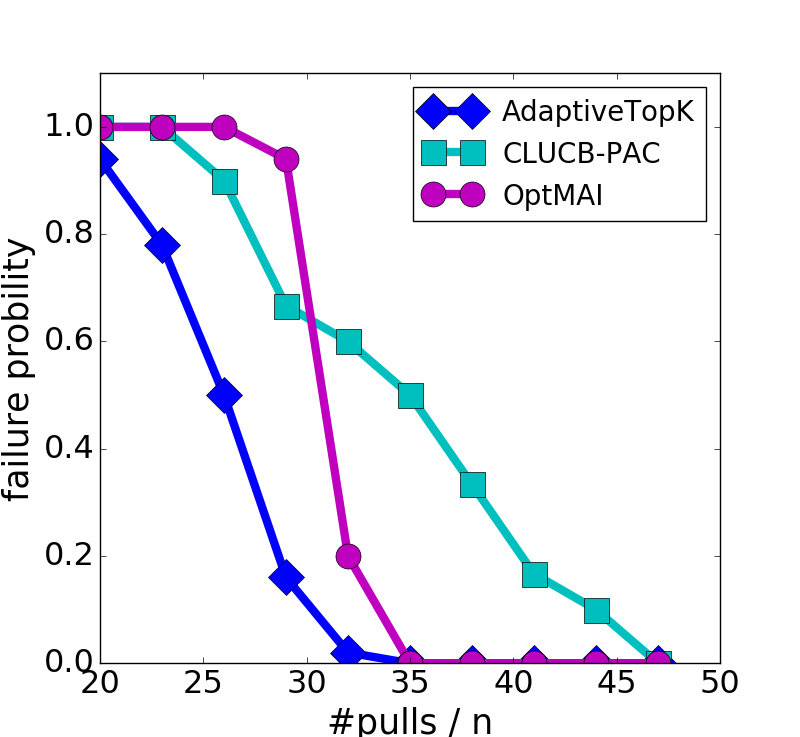}}
     \subfloat[][$K=500$]{\includegraphics[height=0.19\textwidth,width=0.25\textwidth]{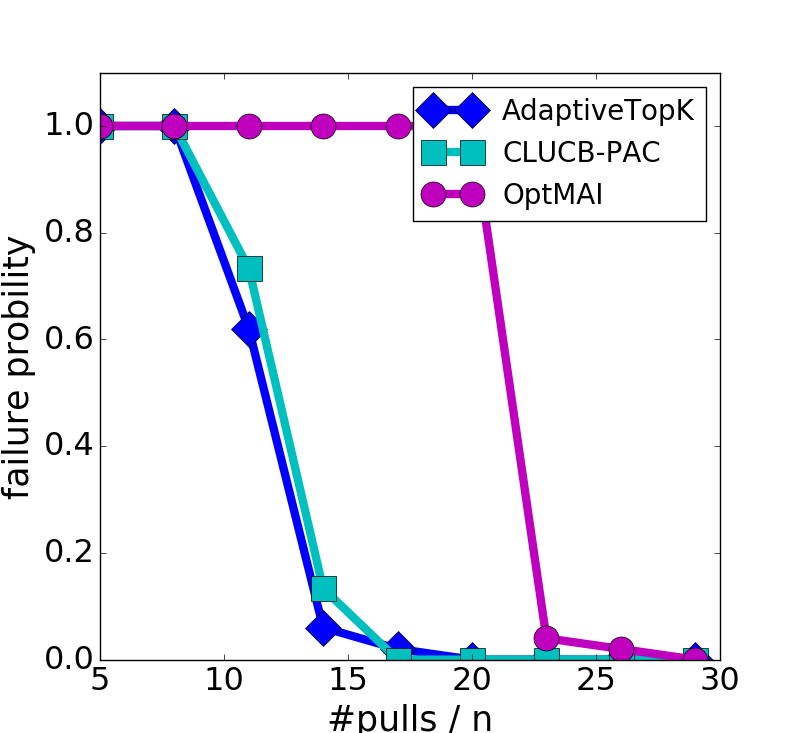}}
     \caption{\synhalf~dataset}
     \label{fig:synhalf}
\end{figure*}

\begin{figure*}[!ht]
     \centering
     \subfloat[][$K=100$]{\includegraphics[height=0.19\textwidth,width=0.25\textwidth]{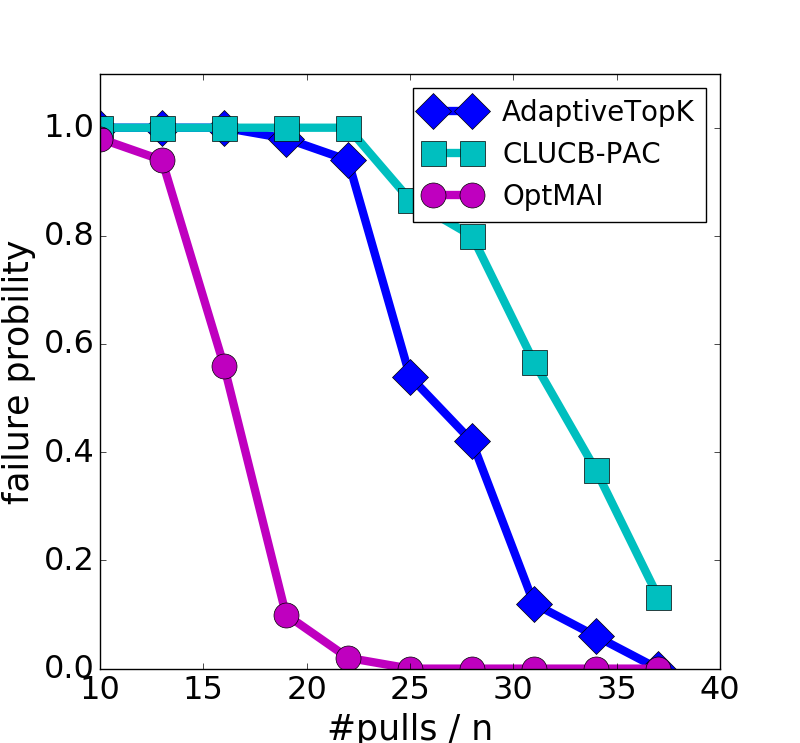}}
     \subfloat[][$K=250$]{\includegraphics[height=0.19\textwidth,width=0.25\textwidth]{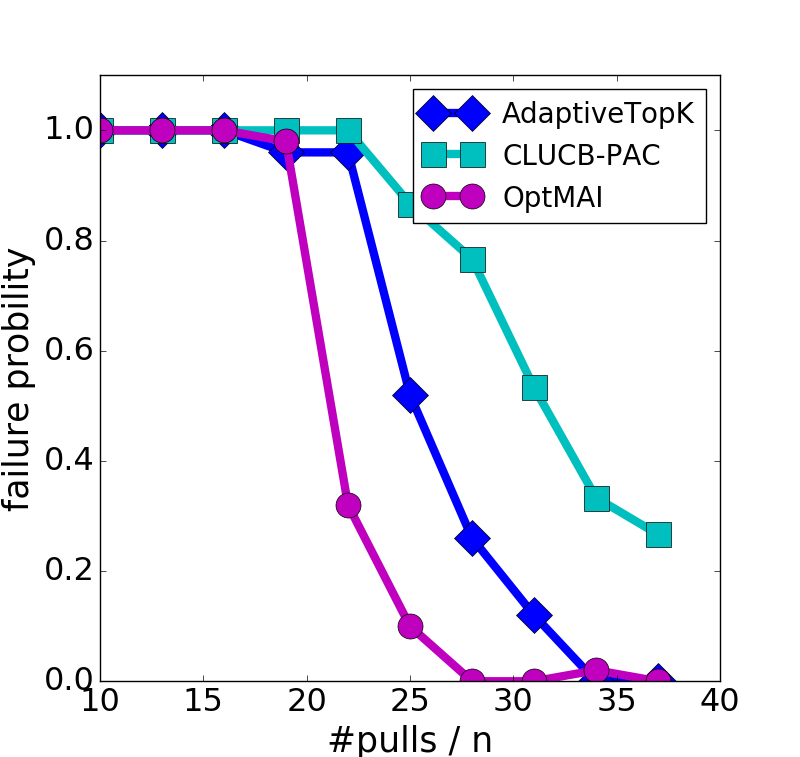}}
     \subfloat[][$K=500$]{\includegraphics[height=0.19\textwidth,width=0.25\textwidth]{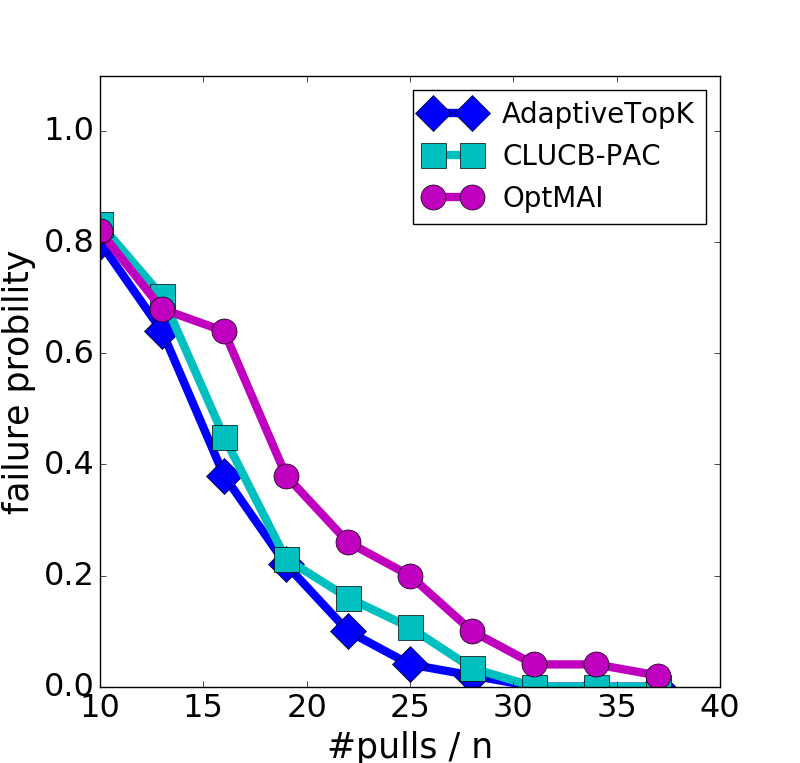}}
     \caption{\uniform~dataset}
     \label{fig:uniform}
\end{figure*}

\begin{figure*}[!ht]
     \centering
     \subfloat[][$K=100$]{\includegraphics[height=0.19\textwidth,width=0.25\textwidth]{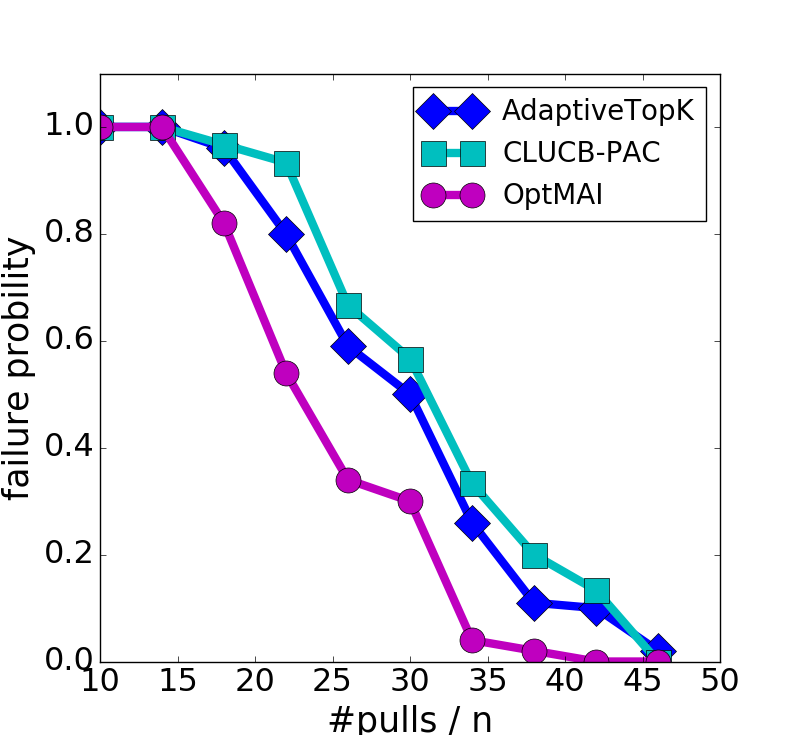}}
     \subfloat[][$K=250$]{\includegraphics[height=0.19\textwidth,width=0.25\textwidth]{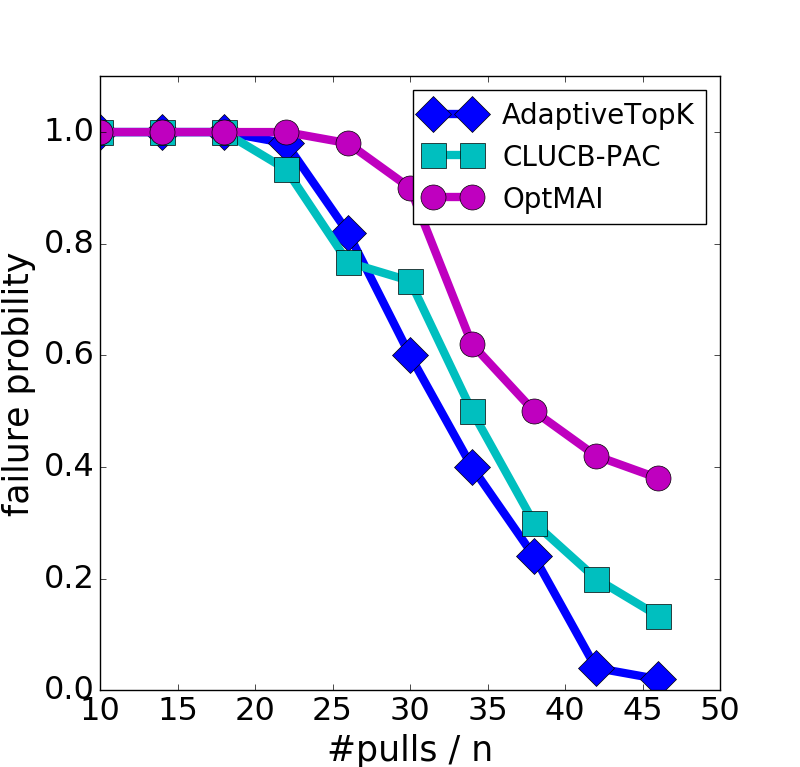}}
     \subfloat[][$K=500$]{\includegraphics[height=0.19\textwidth,width=0.25\textwidth]{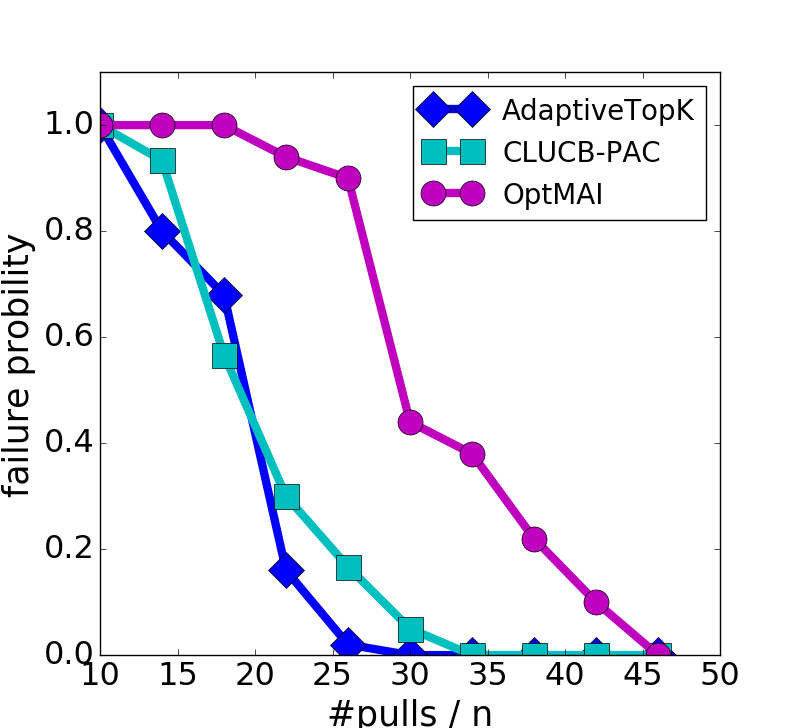}}
     \caption{\synsix~dataset}
     \label{fig:synsix}
\end{figure*}

\begin{figure*}[!ht]
     \centering
     \subfloat[][$K=30$]{\includegraphics[height=0.19\textwidth,width=0.25\textwidth]{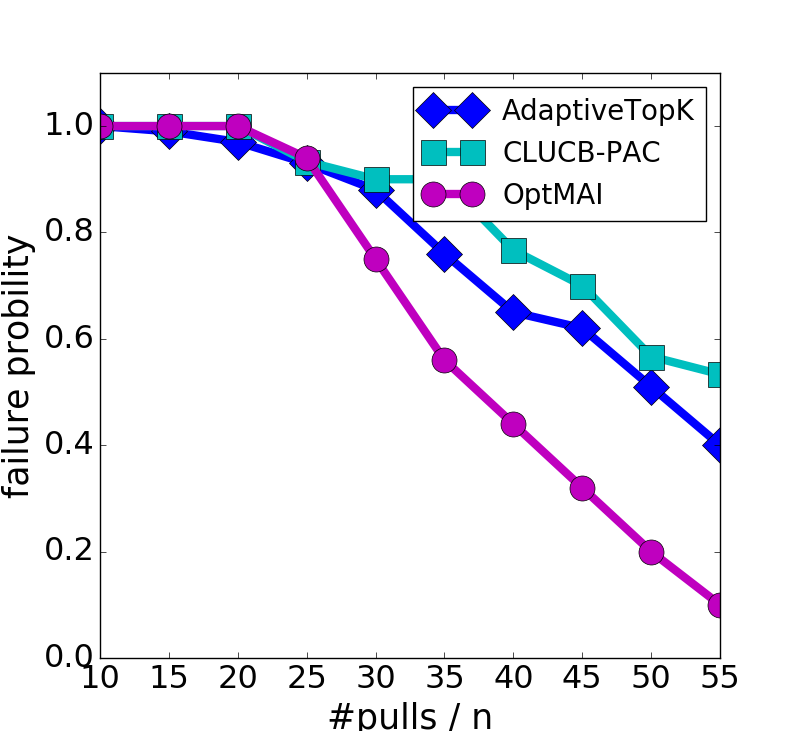}}
     \subfloat[][$K=50$]{\includegraphics[height=0.19\textwidth,width=0.25\textwidth]{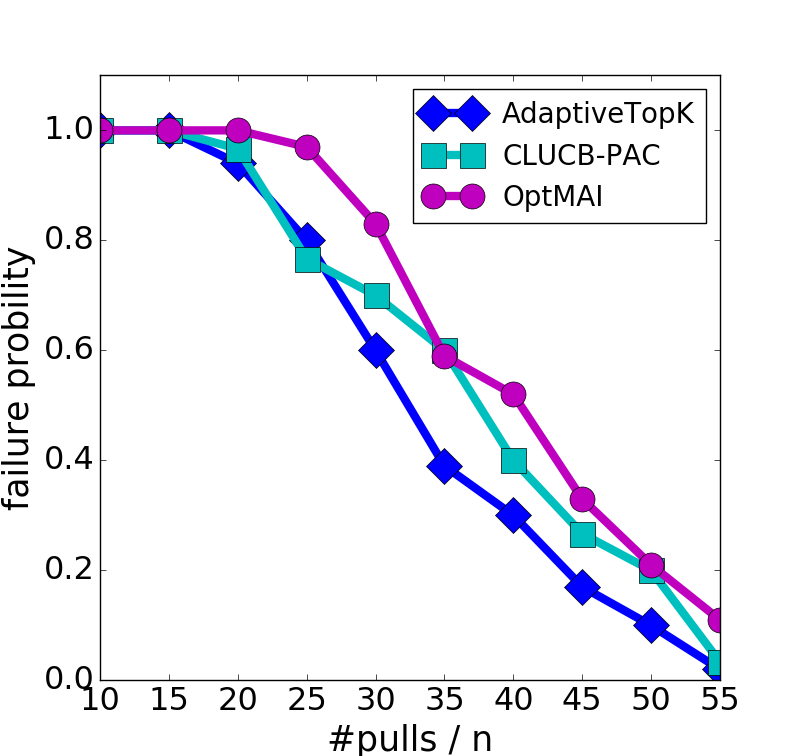}}
     \subfloat[][$K=80$]{\includegraphics[height=0.19\textwidth,width=0.25\textwidth]{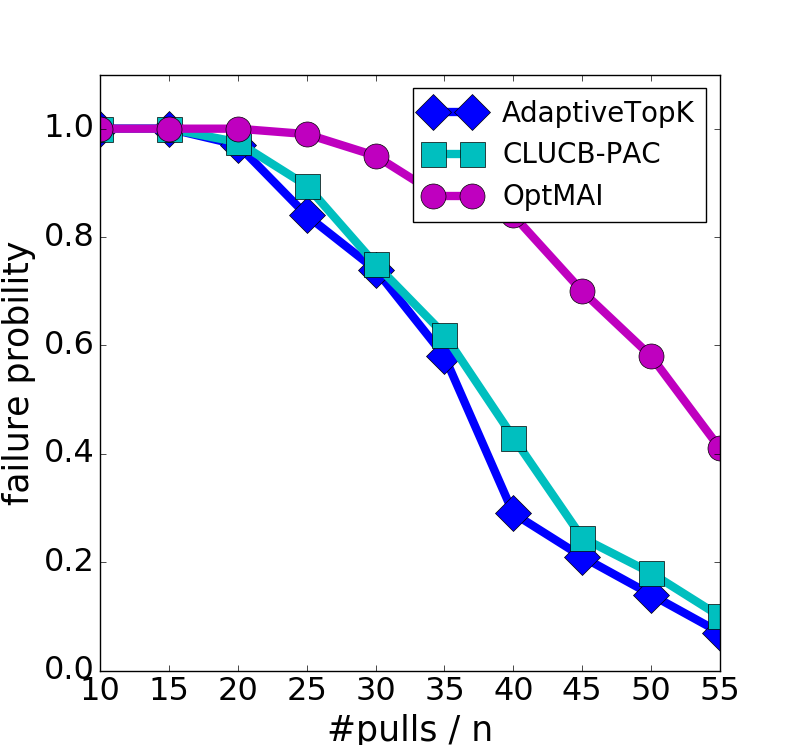}}
     \caption{\rte~dataset}
     \label{fig:rte}
\end{figure*}

It can be observed from the experimental results that \adptopk~(Algorithm \ref{alg:basic}) outperforms \clucb~in almost all the datasets. When $K$ is relatively small, \optmai~has the best performance in most datasets. When $K$ is large, \adptopk~outperforms \optmai~.  The details of  the experimental results are elaborated as follows.
\begin{itemize}
\item For \twogroup~dataset (see Figure~\ref{fig:twogroup}), \adptopk~outperforms other algorithms significantly for all values of $K$. The advantage comes from the adaptivity of our algorithm. In the \twogroup~dataset, top-$K$ arms are very well separated from the rest. Once our algorithm identifies this situation, it need only a few pulls to classify the arms. In details, the inner while loop (Line \ref{line:while-start}) of Algorithm \ref{alg:basic} make it possible to accept/reject a large number of arms in one round as long as the algorithm is confident.

\item As $K$ increases, the advantage of \adptopk~over other algorithms (\optmai~in particular)  becomes more significant. This can be explained by the definition of $H^{(t, \epsilon)}$: $t = t(\epsilon, K)$ usually becomes bigger as $K$ grows, leading to a smaller hardness parameter $H^{(t, \epsilon)}$.


\item A comparison between \synhalf, \uniform, \synsix~reveals that the advantage of \adptopk~over other algorithms (\optmai~in particular) becomes significant in both extreme scenarios, i.e., when arms are very well separated ($p \ll 1$) and when arms are very close to the separation boundary ($p \gg 1$).
\end{itemize}

\section{Conclusion and Future Work}

In this paper, we proposed two algorithms for a PAC version of the multiple-arm identification problem in a stochastic multi-armed bandit (MAB) game. We introduced a new hardness parameter for characterizing the difficulty of an instance when using the aggregate regret as the evaluation metric, and established the instance-dependent sample complexity based on this hardness parameter. We also established lower bound results to show the optimality of our algorithm in the worst case. Although we only consider the case when the reward distribution is supported on $[0,1]$, it is straightforward to extend our results to sub-Gaussian reward distributions.

For future directions, it is worthwhile to consider more general problem of pure exploration of MAB under matroid constraints, which includes the multiple-arm identification as a special case, or other polynomial-time-computable combinatorial constraints such as matchings. It is also interesting to extend the current work to finding top-$K$ arms in a linear contextual bandit framework.

\section*{Appendix}

\subsection*{Proof of Lemma \ref{lem:sum}}
\begin{proof}
We partition all $a_i$'s to groups $G_1, \ldots, G_{\lceil \log_2 M\rceil}$ where $G_j = \{i \in [n]\ |\ a_i \in [2^{j-1}, 2^j) \}$.  Let  $S = \sum_{i = 1}^{n} a_i \geq n$. For each $j \in \{1, 2, \dots, \lceil \log_2 M \rceil\}$, let $\delta_j = \frac{1}{S} \sum_{i \in G_j} a_i$. Observe that we have $\sum_j \delta_j = 1$. 
We will show for each $G_j$ that

\begin{align}
\sum_{i \in G_j} a_i \log(n/i) \leq O\left(\delta_j  \log M + \delta_j \log \delta_j^{-1} + 1 \right)S. \label{eq:sum-1pre}
\end{align}

Once we establish \eqref{eq:sum-1pre}, we prove the lemma as follows. We sum up the inequalities for all $j$ and get 
\begin{align*}
& \sum_{i \in [n]} a_i \log (n/i) \leq \sum_{j = 1}^{\lceil \log_2 M\rceil} O\left(\delta_j  \log M + \delta_j \log \delta_j^{-1} + 1 \right)S \\
 &  \qquad \qquad = O(\lceil \log M\rceil) S + O(S) \sum_{j = 1}^{\lceil \log_2 M\rceil}   \delta_j \log \delta_j^{-1}  \leq O(\lceil \log M \rceil) S + O(S) \cdot \log \lceil \log_2 M\rceil \leq O(\lceil\log M\rceil) S ,
\end{align*}
where the second last inequality is by Jensen's inequality and the convexity of $\delta \log \delta$ over $\delta \in (0, 1)$. 

Now we prove \eqref{eq:sum-1pre} for each group $G = G_j$ and $\delta = \delta_j$. Let $b = \max_{i \in G} \{ a_i\} \leq M$. By our partition rule we have that $b \leq 2 a_i$ for all $i \in G$. Observe that
\begin{align}
\sum_{i \in G}&  a_i \log (n/i) \leq  b \sum_{i=1}^{|G|} \log (n/i)   =  b |G| \log n - b \log (|G|!) \leq b |G| \log (n/|G|) + O(b |G|).
\label{eq:bound-sum-ai}
\end{align}
The last inequality of (\ref{eq:bound-sum-ai}) is by Stirling's approximation. Since $b|G| \geq \sum_{i \in G} a_i = \delta S$, we have $|G| \geq \frac{\delta S}{b} \geq \frac{\delta n}{b}$. We finish the proof of \eqref{eq:sum-1pre} by upper-bounding the RHS of \eqref{eq:bound-sum-ai} by
\[
b|G| \log (b \delta^{-1}) + O(b|G|) \leq 2 \delta S \log (b \delta^{-1})  + O(b|G|) \leq O\left(\delta S \log (M \delta^{-1}) + S\right),
\]
where the first inequality is because $b |G| \leq \sum_{i \in G} 2 a_i = 2 \delta S$ .

\end{proof}

\bibliographystyle{abbrvnat}
\bibliography{paper}

\end{document}